\documentclass{article}

\PassOptionsToPackage{numbers,compress}{natbib}
\usepackage[preprint]{neurips_2026}

\usepackage[utf8]{inputenc} % allow utf-8 input
\usepackage[T1]{fontenc}    % use 8-bit T1 fonts
\usepackage{url}            % simple URL typesetting
\usepackage{booktabs}       % professional-quality tables
\usepackage{amsfonts}       % blackboard math symbols
\usepackage{amsmath}
\usepackage{amsthm}         % theorem/proof environments (load after amsmath)
\usepackage{nicefrac}       % compact symbols for 1/2, etc.
\usepackage{microtype}      % microtypography
\usepackage{xcolor}         % colors
\usepackage{graphicx}
\usepackage{soul}
\usepackage{hyperref}       % hyperlinks (load near the end)
\usepackage{enumitem}
\usepackage{cleveref}       % smart references (must come after hyperref)
\usepackage{wrapfig}
\usepackage{booktabs}   % per \toprule, \midrule, bottomrule  

\usepackage{multirow}
\usepackage{placeins}

\newtheoremstyle{break}
  {3pt}   % space above
  {3pt}   % space below
  {\itshape} % body font
  {}      % indent amount
  {\bfseries} % theorem head font
  {.}     % punctuation after theorem head
  {\newline} % space after theorem head
  {}      % theorem head spec
\theoremstyle{break}
\newtheorem{theorem}{Theorem}

\theoremstyle{definition}
\newtheorem{lemma}[theorem]{Lemma}
\newtheorem{observation}[theorem]{Observation}
\newtheorem{corollary}[theorem]{Corollary}
\newtheorem{assumption}[theorem]{Assumption}
\newtheorem{proposition}[theorem]{Proposition}

\newtheorem{definition}[theorem]{Definition}

\theoremstyle{remark}
\newtheorem*{remark}{Remark}

\title{How Neural Losses Shape VAE Latents}

\author{%
Giorgio Strano$^{1}$\thanks{Equal contribution. Correspondence to: \texttt{strano@di.uniroma1.it}.}
\And
Luca Cerovaz$^{1,2,\ast}$
\And
Michele Mancusi$^{3,\ast}$
\AND
Tommaso Mencattini$^{4}$
\qquad
Emanuele Rodolà$^{1,2}$
\\[3.5ex]
$^{1}$Sapienza University of Rome
\qquad
$^{2}$Paradigma, Inc.
\qquad
$^{3}$Moises Systems, Inc.
\qquad
$^{4}$EPFL
} 

\begin{document}

\maketitle

\begin{abstract}
Modern VAEs are rarely trained with the pointwise likelihood implied by the standard $\beta$-VAE objective. In practice, pointwise reconstruction is often combined with perceptual and adversarial losses, despite a lack of understanding of how this changes the latent dynamics of the model. 
We show that the choice of reconstruction loss reshapes the rate-distortion problem itself, altering both the information content and the geometry of the learned latent space in ways that may be invisible from reconstructions alone. First, we prove and verify empirically that augmenting pointwise reconstruction with neural terms, such as perceptual and adversarial objectives, reduces the amount of information stored in the latent representations. 
Second, we show that neural reconstruction losses systematically change the geometry of the latent space: they make representations more isotropic and distribute uncertainty more evenly across latent dimensions, producing different posterior variance profiles.
These findings highlight how the rate-distortion tradeoff is not a comprehensive lens to understand the behavior of VAEs, and we propose a more mechanistic approach to investigate how the choice of a distortion metric reshapes the optimization problem.
\end{abstract}

\section{Introduction}

Training a VAE implies navigating a tradeoff between reconstruction quality and closeness of the latent space to a simple prior distribution. This is typically done by maximizing the evidence lower bound (ELBO), a two-term objective that balances reconstruction quality against a regularizer (the KL loss), which constrains the latent representation to remain close to a standard Gaussian. This objective admits a natural \textit{rate-distortion} (RD) interpretation \cite{DBLP:conf/icml/AlemiPFDS018}: the reconstruction term acts as a {\em distortion} measure, while the KL divergence to a unit Gaussian acts as an information budget ({\em rate}) that limits how much instance-specific information can be stored in the latent variables. 

By adjusting the weight $\beta$ of the KL term \cite{higgins2017betavae}, one can navigate the RD tradeoff. This process maps out a Pareto-optimal frontier: a boundary where any further reduction in distortion requires a higher rate, and any reduction in rate inevitably increases distortion. From this perspective, $\beta$ serves as a control for data compression, allowing to tune the balance between reconstruction fidelity and latent compactness. 

The standard $\beta$-VAE formulation for a diagonal-Gaussian posterior and Gaussian prior requires a squared error distortion metric. However, in practice, this is almost never the case: most real-world VAEs are trained with a reconstruction loss that is augmented with a complex mixture of perceptual and discriminative neural losses. Because of this, training VAEs remains more an art than a science. 

In this paper, we propose to re-examine from the ground up the role of distortion metrics. Rather than treating them as static implementation choices, we frame the distortion function as a {\em primary} variable in the rate-distortion tradeoff. In particular, we investigate how shifting from pointwise to perceptual and adversarial losses reshapes the Pareto frontier of attainable solutions and the internal geometry of the resulting latent distributions. 

Our main contributions are as follows:

\textbf{Neural losses lower the amount of information in the latents.} We show that common perceptual and adversarial objectives are weaker distortion measures than pixel-wise squared error: after appropriate normalization, their ideal rate-distortion problems require lower or equal rate than the squared error baseline. We then confirm the same qualitative effect in trained VAEs: across losses such as LPIPS \cite{zhang2018unreasonableeffectivenessdeepfeatures}, DINOv2 features \cite{oquab2023dinov2}, and PatchGAN-based adversarial objectives \cite{isola2018imagetoimagetranslationconditionaladversarial}, increasing the neural component of the reconstruction loss consistently reduces the KL reached at convergence, meaning that the learned latent representations store less information.
% Common perceptual and adversarial objectives are weaker distortion measures than pixel-wise squared error in the rate-distortion sense, and therefore admit a lower optimal rate. We prove this formally and verify it empirically: losses such as LPIPS \cite{zhang2018unreasonableeffectivenessdeepfeatures}, DINOv2 features \cite{oquab2023dinov2}, PatchGAN hinge loss \cite{isola2018imagetoimagetranslationconditionaladversarial} reduce the KL divergence reached by the model at convergence, storing less information in the latent representations. 

% We formally prove that neural losses lower the amount of information stored in the latent representations. In practice, this means that all commonly used perceptual and adversarial losses (LPIPS \citep{zhang2018unreasonableeffectivenessdeepfeatures}, DINO features \citep{oquab2023dinov2}, PatchGAN Hinge Loss \citep{isola2018imagetoimagetranslationconditionaladversarial}) lower the optimal rate (KL divergence) reached by the model at convergence.

\textbf{Distortion shapes latent uncertainty.} We show that rate constrains the amount of information but not its allocation: we focus on the individual latent representations, and demonstrate that a fixed information rate does not uniquely determine how uncertainty is distributed across latent dimensions. Instead, the choice of distortion function governs the {\em uncertainty profile} of the posterior (Figure~\ref{fig:anisotropy}): while pointwise squared error concentrates precision into a few highly anisotropic dimensions, we observe that, in practice, perceptual and adversarial losses encourage a more uniform, isotropic distribution of variance across the posterior representation.

\begin{figure}[t]
  \centering
  \includegraphics[width=0.8\linewidth]{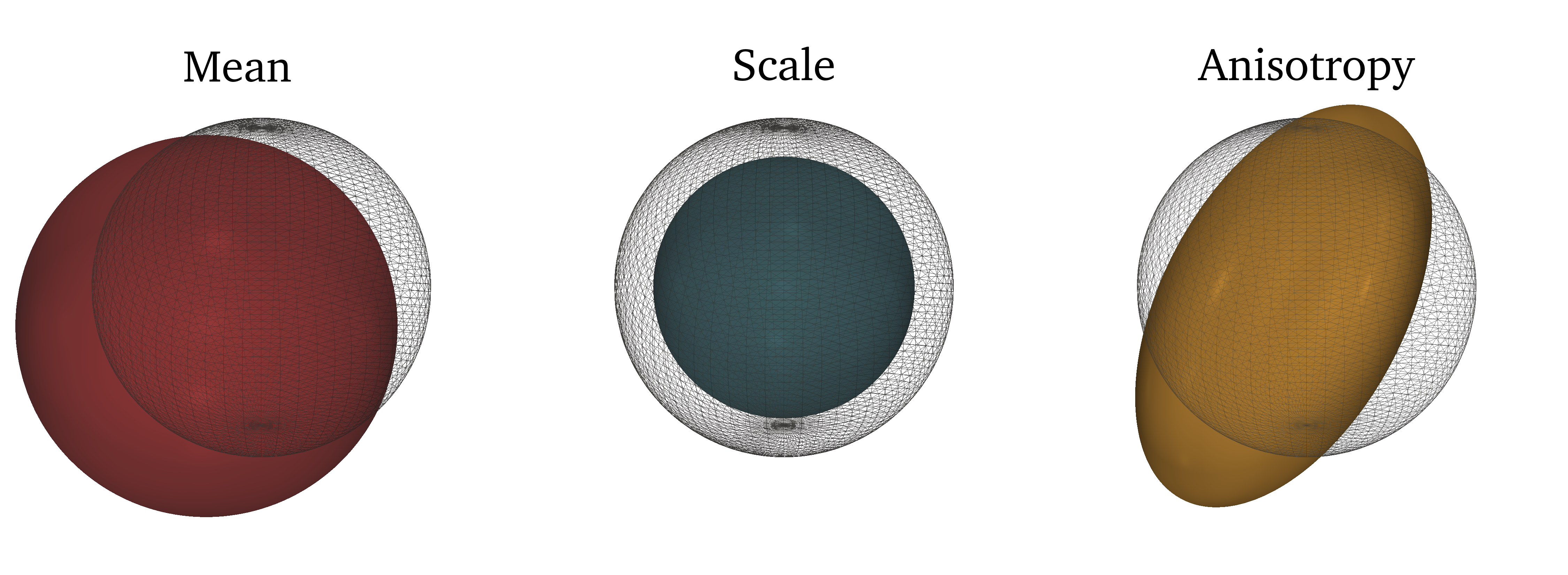}
  \caption{Geometric decomposition of the VAE KL term. The gray wireframe is the prior $\mathcal{N}(0, I)$; the colored surface is an iso-density contour of the posterior $\mathcal{N}(\mu, \Sigma)$.}

  % \caption{\textbf{TODO: add ref, edit caption} \textbf{Geometric decomposition of the VAE KL term.} The gray wireframe denotes the prior $p(z)=\mathcal{N}(0,I)$, while the colored surface shows an iso-density contour of the posterior $q(z|x)=\mathcal{N}(\mu,\Sigma)$. \emph{Mean:} pure translation of the posterior center. \emph{Scale:} isotropic expansion/contraction. \emph{Anisotropy:} shape change at fixed volume.}
  \label{fig:anisotropy}
\end{figure}

\section{Background}
\label{sec:background}

This section briefly reviews the VAE objective and the rate-distortion lens that motivates our research.

\paragraph{Variational autoencoders.}
\label{sec:vae}

VAEs~\cite{DBLP:journals/corr/KingmaW13} are latent-variable generative models that pair a decoder $p_\theta(x \mid z)$ with an amortized variational posterior $q_\phi(z \mid x)$. We use the standard isotropic Gaussian prior and diagonal-Gaussian encoder:
\[
p(z) = \mathcal{N}(0, I_D), \qquad q_\phi(z \mid x) = \mathcal{N}\big(\mu_\phi(x), \mathrm{diag}(\sigma^2_\phi(x))\big).
\]

\paragraph{ELBO and the $\beta$-VAE Objective.}
\label{sec:elbo}

Since the marginal likelihood $p_\theta(x)$ is intractable for neural decoders, VAEs optimize the evidence lower bound. We work with the standard $\beta$-VAE variant~\cite{higgins2017betavae}, written in minimization form as
\begin{equation}
\mathcal{L}_\beta := \underbrace{\mathbb{E}_{x \sim p_\mathcal{D}} \mathbb{E}_{z \sim q_\phi(z \mid x)} \big[-\log p_\theta(x \mid z)\big]}_{D \text{ (distortion)}} + \beta \underbrace{\mathbb{E}_{x \sim p_\mathcal{D}} \big[\mathrm{KL}(q_\phi(z \mid x) \,\|\, p(z))\big]}_{R \text{ (rate)}}.
\label{eq:beta-vae}
\end{equation}
The first term encourages accurate reconstructions; the second limits the information stored in $z$ by pushing the posterior toward the prior. Together, they create a tradeoff bounded by the rate-distortion curve; $\beta$ controls how the model navigates this frontier.  Figure~\ref{fig:rd-curve} shows the exact RD curve obtained by training the same model at different $\beta$ values. 

\paragraph{Rate-distortion interpretation.}
\label{sec:rd}

Shannon rate-distortion theory~\cite{6773024} characterizes the minimum mutual information required to reconstruct a source under a prescribed expected distortion budget. Its Lagrangian relaxation mirrors the $\beta$-VAE objective in Eq.~\eqref{eq:beta-vae}, motivating the interpretation of $\beta$ as a compression knob~\cite{DBLP:conf/icml/AlemiPFDS018} that traces an approximate, parametric RD frontier.

\begin{wrapfigure}[8]{r}{0.45\textwidth}
  \vspace{-0.5em}
  \centering
  \includegraphics[width=0.4\textwidth]{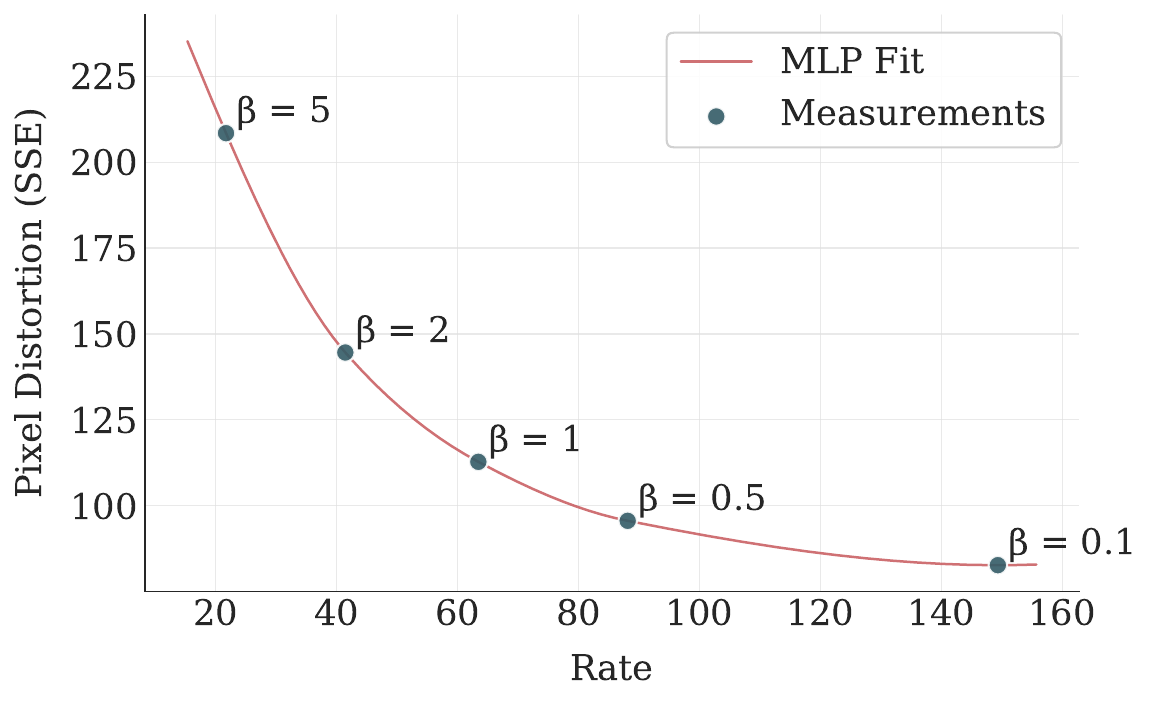}
  \caption{Rate-Distortion curve traced by $\beta$.}
  % \caption{The rate-distortion curve obtained by training the same model with the ELBO objective at different values of $\beta$. Higher $\beta$ implies lower rate at the cost of higher distortion; lower $\beta$ allows lower distortion at the cost of more information stored in the latents.}
  \label{fig:rd-curve}
  % \vspace{-1em}
\end{wrapfigure}

If the decoder likelihood is Gaussian with fixed variance, the distortion term reduces to scaled squared error $\|x - p(x | z)\|_2^2$. More broadly, a chosen likelihood family induces a particular reconstruction loss. In modern practice, this term is often augmented or replaced with perceptual or adversarial losses, breaking the standard pixel-squared-error assumption that Section~\ref{sec:claim1} re-examines.
\\ 

% \vspace{20pt}

\section{Neural losses lower the optimal rate}
\label{sec:claim1}

Modern VAEs rarely rely on pure pointwise reconstruction losses. Instead, they often augment or replace pixel losses with perceptual or adversarial objectives. 
In this section we study, both theoretically and empirically, how these neural objectives influence the optimal \emph{rate} of the VAE: the amount of information stored in the latent representations.
% In this section we separate two related claims. 
% First, at the Shannon rate-distortion level, common neural reconstruction terms are weaker than pixel SSE: after accounting for scale and additive constants, their induced RD curves are no larger than the pixel-SSE RD curve. Second, we verify empirically that the same effect appears in finite VAE training: when loss magnitudes are normalized and \(\beta\) is held fixed, increasing the neural component of the reconstruction loss systematically lowers the KL rate achieved at convergence.

\subsection{Rate-distortion analysis}\label{sec:claim1-theory}
% We focus our proof on two specific losses: MSE between features extracted from a pretrained VGG (as a perceptual example) and PatchGAN Hinge Loss (as a discriminative one). We explain in appendix \textbf{TODO: CITE} how these proofs easily extend to most other neural losses like LPIPS, DINO features, and similar.

% We proceed by:
%     \begin{enumerate}[label=(\roman*), leftmargin=20pt, topsep=0pt]
%         \item observing how a weaker distortion (under point-wise affine domination) lowers the RD curve;
%         \item proving that a perceptual loss based on VGG features is weaker than squared error;
%         \item proving that a discriminative loss based on a GAN network is weaker than squared error.
%     \end{enumerate}
% We prove the Shannon-level statement in three steps. We first introduce a simple sufficient condition under which one distortion has an RD curve no larger than another. We then show that feature-space squared error satisfies this condition under a Lipschitz assumption on the feature map. Finally, we apply the same idea to discriminator-based generator losses by converting the generator score into a shifted nonnegative pairwise distortion. The empirical subsection then tests whether this ideal RD ordering is reflected in trained VAEs.
In this section we prove formally that two of the most common neural losses used in VAE training are \emph{weaker} than pixel-squared error. Therefore, they induce \emph{lower RD curves}, which means they favor storing \emph{less information} in the latents. 

We focus on the MSE between VGG features as a perceptual objective, and on hinge loss with PatchGAN \cite{isola2018imagetoimagetranslationconditionaladversarial} as an adversarial one. We explain in Appendix~\ref{app:generalization} how these results easily extend to most neural losses under standard assumptions.

\paragraph{Weakness of distortions and RD ordering.}
We compare reconstruction losses through their induced Shannon RD functions.

Let
\[
R_d(\Delta)
:=
\inf_{P_{\hat X|X}:\,\mathbb{E}[d(X,\hat X)]\le \Delta}
I(X;\hat X)
\]
for a distortion function \(d\) and distortion budget \(\Delta\ge 0\), where the infimum is over
all reconstruction kernels \(P_{\hat X|X}\).

% Let 
% \begin{equation}
% R_d(\Delta):=\inf_{r:\,\mathbb{E}[d(X,\hat X)]\le \Delta} I(X;\hat X) 
% \end{equation}
% for a distortion function $d$ and distortion budget $\Delta\ge 0$.

% This is the ideal information-distortion tradeoff associated with the loss \(d\). In the experiments below we measure the VAE KL term, which is the standard variational proxy for rate; the Shannon RD function is used here because it gives a clean model-independent way to compare distortion functions.

We use pointwise affine domination to formalize the idea that one distortion function is less restrictive than another: if \(d_2\) is always bounded by a scaled and shifted version of \(d_1\), then every reconstruction that is good under \(d_1\) is automatically good under \(d_2\), after changing the budget accordingly.

% \begin{definition}[Weakness as affine domination]
% For $c>0$, $b\ge0$, we say $d_2$ is \emph{weaker} than $d_1$ if it is dominated by $d_1$ up to scale and shift, and write 
% %
% \begin{equation}
% d_2 \preceq_c^b d_1  \quad \textrm{if} \quad d_2(x,\hat x)\le c\,d_1(x,\hat x) + b
% \end{equation}
% %
% for all $(x,\hat x)$.
% \end{definition}

\begin{definition}[Weakness as affine domination]\label{def:weak}
For \(c>0\) and \(b\ge 0\), we say that \(d_2\) is weaker than \(d_1\), and write \(d_2 \preceq_{c,b} d_1\), if
\begin{equation}
d_2(x,\hat{x}) \le c\,d_1(x,\hat{x}) + b
\end{equation}
for all \((x,\hat{x})\). When \(b=0\), we simply write \(d_2\preceq_c d_1\).
\end{definition}

% This definition is intentionally only a sufficient condition. It does not claim that \(d_2\) is pointwise smaller than \(d_1\) on the same numerical scale; rather, it says that \(d_2\) cannot distinguish reconstructions more finely than \(d_1\) once scale and additive slack are accounted for. This pointwise relation will translate directly into an ordering of the corresponding RD functions.

Now comparing distortion functions becomes a simple feasible-set argument. If satisfying a budget under \(d_1\) automatically satisfies the corresponding budget under \(d_2\), then the infimum of mutual information under \(d_2\) cannot be larger.

% This pointwise domination implies that any channel satisfying a budget on $d_1$ automatically satisfies a corresponding budget on $d_2$. Consequently, the information rate required for the weaker metric is upper-bounded by the rate of the stronger:

\begin{observation}[RD ordering under weakness]
% If $d_2 \preceq_c^b d_1$, then for all $\Delta\ge b$,
If \(d_2 \preceq_{c,b} d_1\), then for all \(\Delta \ge b\),
\begin{equation}
% R_{d_2}(\Delta)\le R_{d_1}\left(\frac{\Delta - b}{c}\right).
R_{d_2}(\Delta)
\le
R_{d_1}\!\left(\frac{\Delta-b}{c}\right).
\end{equation}
% Equivalently, defining the normalized distortion $\tilde d_2:=\frac{d_2 - b}{c}$, we have
Equivalently, for every \(\Delta\ge 0\),
\begin{equation}
% R_{\tilde d_2}(\Delta)\le R_{d_1}(\Delta)\qquad \forall \Delta\ge 0.
R_{d_2}(c\Delta+b) \le R_{d_1}(\Delta).
\end{equation}
In the special case \(b=0\), this can also be written as
\begin{equation}
    R_{d_2/c}(\Delta) \le R_{d_1}(\Delta).
\end{equation}

\end{observation}

% \emph{Proof sketch.}
% This is a straightforward conclusion from rate-distortion thoery. The feasible set for $(d_1,\Delta/c)$ is contained in that for $(d_2,\Delta)$, so the infimum of $I(X;\hat X)$ cannot increase.
% Let \(r\) be any reconstruction kernel with \(D(r;d_1)\le(\Delta-b)/c\). Since \(d_2(x,\hat{x})\le c\,d_1(x,\hat{x})+b\) pointwise, we have \(D(r;d_2)\le \Delta\). Thus the feasible set for \((d_1,(\Delta-b)/c)\) is contained in the feasible set for \((d_2,\Delta)\). Taking the infimum of \(I(X;\hat X)\) over the larger feasible set cannot increase the value.
% (Full proof, and discussion on when this bound can be strict or tight in \cref{app:claim1:core}.)

Thus weakening the distortion cannot increase the Shannon-optimal rate. 
% The inequality is non-strict: domination alone guarantees that the weaker RD curve lies below or on the stronger one, but strict separation requires the weaker loss to ignore variation that the stronger loss must encode. The next results show that common neural losses fall into this weaker-distortion class.

% This ordering implies that for a fixed latent capacity (or fixed $\beta$), a weaker distortion demands significantly less information from the latent bottleneck. The model is therefore prone to under-utilizing the available dimensions, a phenomenon we link to the emergence of anisotropy.

Now we apply the ordering result to perceptual feature losses. The intuition is that a Lipschitz feature extractor cannot amplify an input perturbation by more than a constant factor. Therefore, pixel-level closeness implies feature-level closeness, whereas feature-level closeness may still allow many pixel-level differences.

\begin{theorem}[VGG feature squared error is dominated by pixel squared error] 

Let pixel squared error \(d_{\rm pix}\) and a perceptual distortion \(d_\phi\) be
\[
d_{\rm pix}(x,\hat{x}) := \|x-\hat{x}\|_2^2,
\qquad
d_{\phi}(x,\hat{x}) := \|\phi(x)-\phi(\hat{x})\|_2^2 .
\]
Assume that the feature map \(\phi\) is \(L\)-Lipschitz on the admissible image domain \(\mathcal X\). Then
\[
d_{\phi}\preceq_{L^2} d_{\rm pix}.
\]
% FORMAL WRITING OF LIPSCHITZNESS - CAN REMOVE
% \[
% \|\phi(x)-\phi(\hat{x})\|_2 \le L\|x-\hat{x}\|_2
% \qquad
% \forall x,\hat{x}\in\mathcal X .
% \]
Consequently, for all \(\Delta\ge 0\),
\[
R_{d_{\phi}/L^2}(\Delta)
\le
R_{d_{\rm pix}}(\Delta).
\]

\end{theorem}

% OLD PROOF SKETCH, TOO UNREADABLE
% \emph{Proof sketch.}
% The VGG feature map $\phi$ is Lipschitz on the data domain since it is a composition of Lipschitz layers (convolutions/affine maps, 1-Lipschitz activations, and nonexpansive pooling), hence $\text{Lip}(\phi)$ is bounded by a product of layerwise operator-norm bounds \cite{DBLP:journals/ml/GoukFPC21, DBLP:conf/icml/DelattreBAA23, DBLP:journals/simods/CombettesP20}.
% Square the Lipschitz bound to get $\|\phi(x)-\phi(\hat x)\|_2^2\le L^2\|x-\hat x\|_2^2$ pointwise, i.e.
% $d_\phi \preceq_{L^2} d_{\mathrm{pix}}$, then apply the RD ordering lemma.
% (Full proof in \cref{app:claim1:vgg}.)

\emph{Proof sketch.}
A VGG feature map is Lipschitz on the image domain because it is a composition
of Lipschitz layers. Therefore, pixel-close reconstructions are also close in
feature space, and squaring this Lipschitz bound gives pointwise domination of
feature squared error by pixel squared error. (Full proof in Section~\ref{app:claim1:vgg}).

In the following theorem, we prove that an adversarial loss based on the score of a Lipschitz discriminator network, such as PatchGAN, can be rewritten as a valid distortion metric that is also weaker than pixel squared error.

% \begin{theorem}[PatchGAN hinge loss is dominated by pixel squared error]
\begin{theorem}[PatchGAN generator loss is dominated by pixel squared error]
Let \(\mathcal X\) be bounded and let \(s_\psi:\mathcal X\to\mathbb R\) be a
fixed Lipschitz PatchGAN score. Then the corresponding generator-side score loss
admits an equivalent nonnegative pairwise RD distortion \(d_\psi\) that
is affinely dominated by pixel SSE:
\[
d_\psi \preceq_{\alpha,\beta} d_{\rm pix}
\]
for every \(\alpha>0\) and some finite \(\beta\). Hence
\[
R_{d_\psi}(\Delta)
\le
R_{d_{\rm pix}}\!\left(\frac{\Delta-\beta}{\alpha}\right),
\qquad
\Delta\ge\beta .
\]
\end{theorem}

\emph{Proof sketch.}
For a fixed PatchGAN discriminator, the generator loss can be rewritten, up to constants, as a nonnegative pairwise score difference between \(x\) and \(\hat x\). Lipschitzness then implies that this score difference is dominated by pixel distance, and therefore by pixel SSE up to an additive constant. (Full proof in Appendix~\ref{app:claim1:disc}).

\begin{corollary}[Mixtures with pixel SSE remain dominated by pixel SSE]
Let \(d_N\preceq_{c,b} d_{\rm pix}\) be any neural distortion satisfying affine domination by pixel SSE, and define
\[
d_\lambda(x,\hat{x})
:=
(1-\lambda)d_{\rm pix}(x,\hat{x})
+
\lambda d_N(x,\hat{x}),
\qquad
\lambda\in[0,1].
\]
Then
\[
d_\lambda \preceq_{(1-\lambda)+\lambda c,\,\lambda b} d_{\rm pix}.
\]
% Consequently, for all \(\Delta\ge \lambda b\),
% \[
% R_{d_\lambda}(\Delta)
% \le
% R_{d_{\rm pix}}\!\left(
% \frac{\Delta-\lambda b}{(1-\lambda)+\lambda c}
% \right).
% \]
\end{corollary}

This shows that replacing part of pixel SSE by a weaker neural term always results in a distortion function that is weaker than pure pixel squared error. The result does not by itself prove monotonicity of the trained VAE rate as a function of \(\lambda\); it only gives the ideal RD ordering. The monotone decrease in achieved KL is therefore an empirical question, which we test next.

\subsection{Empirical evidence}
\label{subsec:claim1-empirical}

\begin{figure}[t]
  \centering
  % \includegraphics[width=0.49\columnwidth]{img/celeba_pythae/lpips/rd_curve.pdf}
  % \hfill
  % \includegraphics[width=0.49\columnwidth]{img/celeba_pythae/lpips/rate_vs_lambda.pdf}
  \includegraphics[width=\columnwidth]{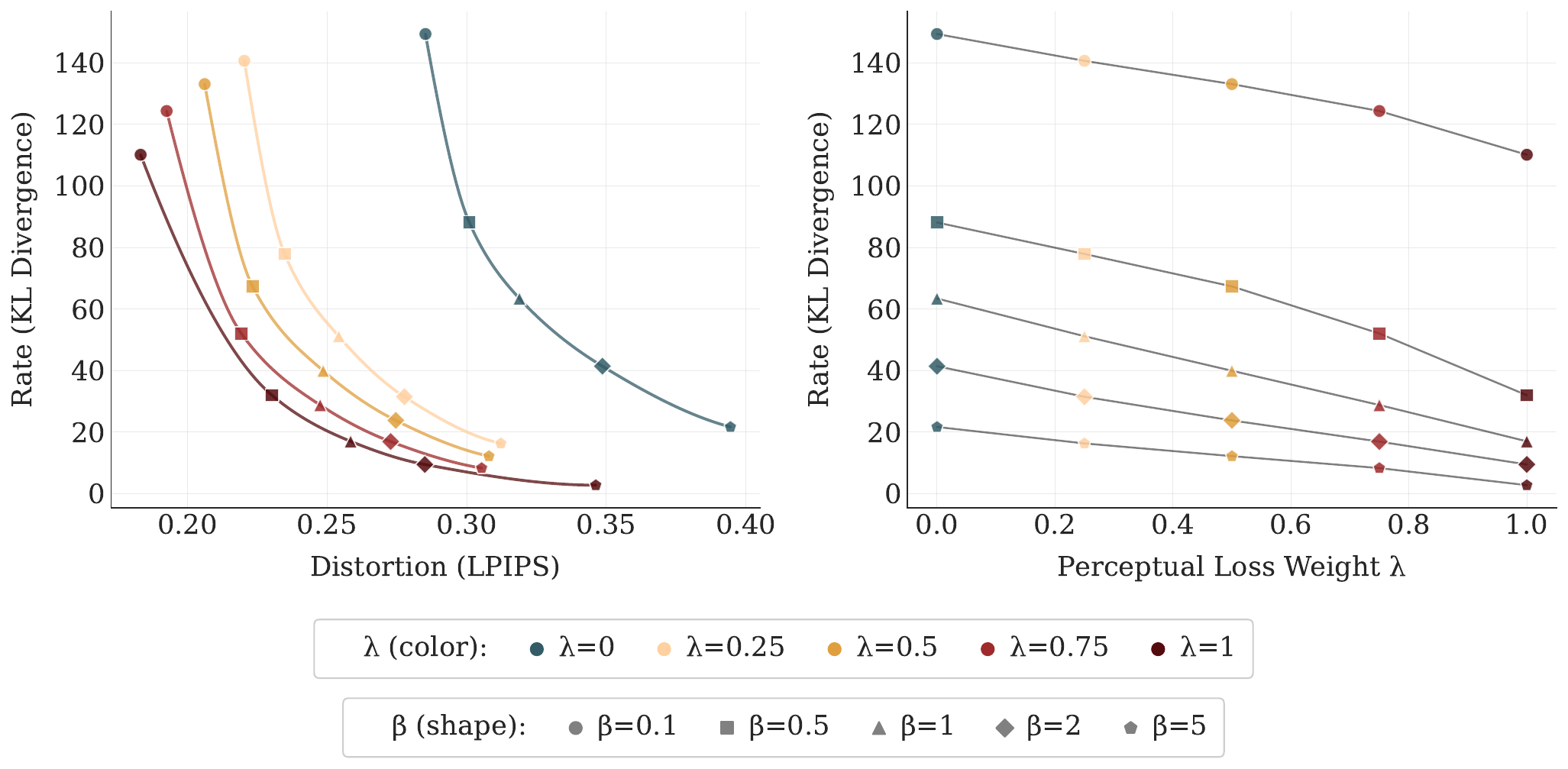}
  \caption{Experimental results supporting the claim in Section~\ref{sec:claim1} for the \texttt{pythae} VAEs trained on CelebA. On the left (RD curves at varying $\lambda$), more perceptual distortion metrics (higher $\lambda$) shift the Pareto-optimality frontier of the rate-distortion tradeoff. On the right, for each fixed $\beta$, the optimal rate shrinks monotonically as the loss becomes more perceptual.
  \label{fig:rd_curves_perceptual}
  }
\end{figure}

The previous section shows that the ideal RD frontier of a distortion that includes perceptual terms lies below that of pixel SSE, but this does not by itself imply that a finite amortized VAE, trained with SGD on a non-convex ELBO, will achieve a lower KL at fixed \(\beta\) as \(\lambda\) grows.

We therefore check this empirically: after normalizing the scale of the reconstruction loss, does increasing the neural component lower the KL rate reached at convergence?

% The preceding arguments compare ideal Shannon RD functions. They do not characterize the nonconvex optimization dynamics of a finite amortized VAE, nor do they guarantee that the achieved KL at a fixed \(\beta\) must decrease monotonically with \(\lambda\). They do, however, predict that neural reconstruction terms make lower-rate solutions feasible by weakening the information requirements of the distortion. We therefore test the operational version of this prediction directly: after normalizing the scale of the reconstruction loss, does increasing the neural component lower the KL rate reached at convergence?

We verify this by training VAEs on the unified training objective:
\begin{equation}
\label{eq:lossfunction}
\mathcal{L} = B(\lambda D + (1 - \lambda) \mathrm{SSE}) + \beta\,\mathrm{KL},
\end{equation}
%
% where $D$ is one of three neural losses: LPIPS \cite{zhang2018unreasonableeffectivenessdeepfeatures} as a perceptual loss, DINOv2-feature MSE \cite{oquab2023dinov2} as a perceptual loss based on a self-supervised backbone, and the PatchGAN hinge loss combined with feature matching \cite{isola2018imagetoimagetranslationconditionaladversarial} as an adversarial loss. 
where $D$ is a neural distortion and 
\(B\) is a loss balancer \cite{defossez2022high} which normalizes the magnitude of the reconstruction term so that $\lambda$ controls the \emph{composition} of the distortion rather than its scale. We sweep $\lambda$ across five $\beta$ values spanning \{0.1, 0.5, 1, 2, 5\} and for each model we record the final KL after training to convergence. The full experimental setup is described in Appendix~\ref{sec:experimental-setup}.

\begin{wrapfigure}[16]{r}{0.45\textwidth}
  \vspace{-0.1em}
  \centering
  \centering
  \includegraphics[width=0.43\columnwidth]{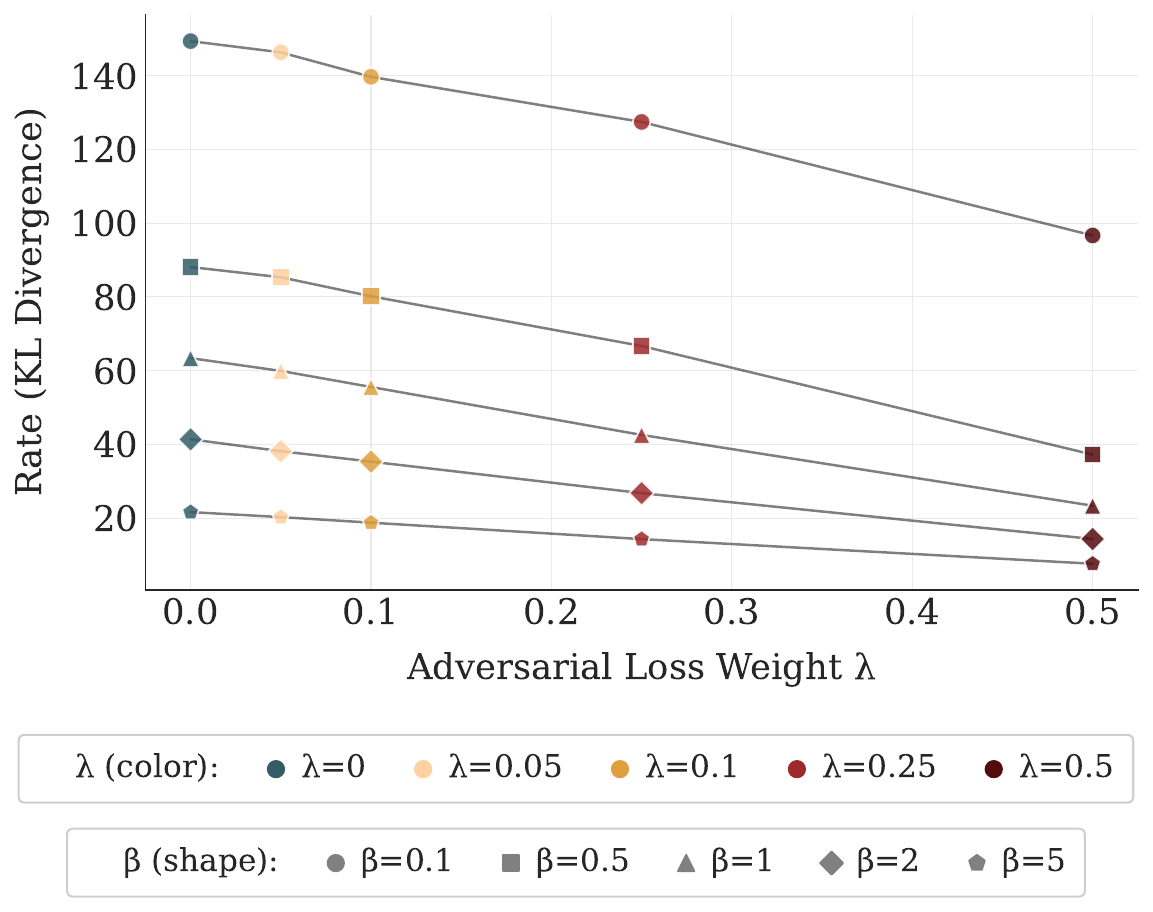}
  \caption{Rate reached at convergence as a function of $\lambda$ for the pythae VAE trained on CelebA with adversarial loss.}
  \label{fig:claim1-adversarial}
\end{wrapfigure}

To test robustness, we vary model architecture (a traditional VAE from \texttt{pythae} \cite{chadebec2022pythae} and AutoencoderKL from \texttt{diffusers} \cite{von-platen-etal-2022-diffusers}), dataset (CelebA \cite{liu2015faceattributes} and Tiny-ImageNet \cite{deng2009imagenet}), and the family of neural distortions (LPIPS \cite{zhang2018unreasonableeffectivenessdeepfeatures} and  DINOv2 features \cite{oquab2023dinov2} as perceptual losses, and a PatchGAN hinge loss with feature matching \cite{isola2018imagetoimagetranslationconditionaladversarial} as an adversarial loss).

The same monotone effect holds across every configuration we tested: as $\lambda$ grows, the Pareto frontier shifts down and models converge to lower rates for the same $\beta$ value. Figure~\ref{fig:rd_curves_perceptual} shows this for the perceptual loss with the pythae VAE on CelebA, and Figure~\ref{fig:claim1-adversarial} for the same setup with the adversarial loss. 
Figure~\ref{fig:autoencoder_kl_claim1} shows results for the AutoencoderKL model on both datasets. 
Appendix~\ref{app:plots} also reports results for DINOv2 (Figure~\ref{fig:autoencoder_kl_claim1-dino}) and the adversarial hinge loss (Figure~\ref{fig:autoencoder_kl_claim1-k}) on Tiny-ImageNet with AutoencoderKL.

% We test our claim on two different model architectures: a traditional VAE architecture as implemented by \texttt{pythae} \cite{chadebec2022pythae} (25.2M params) and \href{https://huggingface.co/docs/diffusers/main/en/api/models/autoencoderkl#diffusers.AutoencoderKL}{AutoencoderKL} from the Diffusers library (14.4M params). 
% We then repeat all our experiments across two datasets: CelebA dataset \cite{liu2015faceattributes} and Tiny-Imagenet \cite{deng2009imagenet}.
% The same monotone effect holds across every configuration we tested: 
% Figure~\ref{fig:rd_curves_perceptual} and Figure~\ref{fig:autoencoder_kl_claim1} report the result for LPIPS. It shows how, as $\lambda$ grows, the Pareto frontier shifts down and models converge to lower rates. 
% The same happens with the adversarial loss as shown in Figure~\ref{fig:pythae_ksweep}.

% Across these settings, increasing the neural component of the reconstruction loss consistently lowers the KL rate reached at convergence. This supports the theoretical prediction that neural distortions induce a lower optimal rate. Section~\ref{sec:claim1} therefore establishes a capacity effect: changing the distortion changes how much information the model stores. The next question is whether distortion still matters after this capacity effect is removed. In Section~\ref{sec:claim2}, we fix the rate and ask whether the choice of distortion changes how the same amount of information is distributed across latent coordinates.

\section{Distortion shapes posterior geometry}
\label{sec:claim2}
Section~\ref{sec:claim1} showed that the choice of distortion function changes the \emph{amount} of information stored in the latent representations: at the same value of $\beta$, neural losses converge to a lower rate than pixel SSE. A natural question now becomes: if we instead \emph{fix the rate as a constant}, does the choice of distortion function alter \emph{how this information is allocated} in the latent representations?

Here we show that the answer is yes. At matched rate, the choice of distortion function changes the \emph{shape of the per-sample posterior}: specifically, how variance is distributed across latent coordinates. Empirically, we observe that perceptual and discriminative losses produce more isotropic posteriors than pixel SSE: information is spread more evenly across coordinates rather than concentrated in a few active dimensions.

\begin{figure}[t]
  \centering
  \includegraphics[width=\columnwidth]{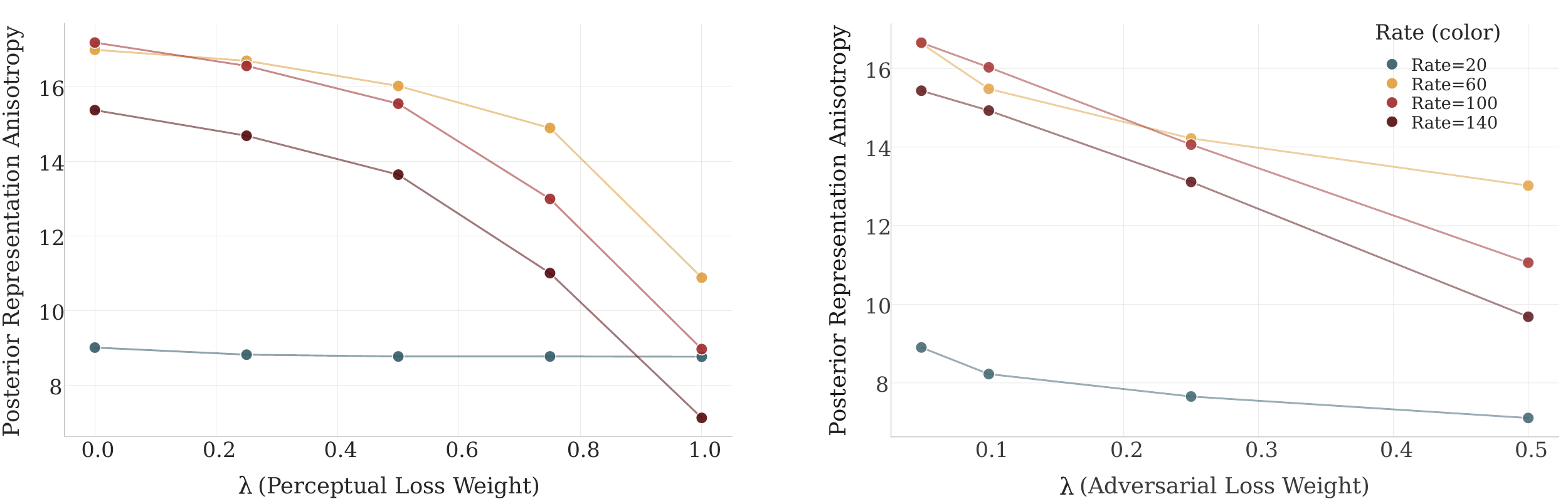}
  \caption{Rate-matched training of the pythae VAE on CelebA, using the perceptual loss LPIPS (left) and adversarial hinge loss with PatchGAN (right). Shows how, for each fixed rate (color), the average per-sample posterior representations become less anisotropic as the neural contribution to the reconstruction loss grows. In both cases, when the rate is already extremely low (blue line at the bottom), the representations are already extremely isotropic, and the effect of the neural loss shrinks. 
  % The target rates were picked from a range that would be naturally reachable with commonly used $\beta$ coefficients.
  \label{fig:claim2-ratematched}
  }
\end{figure}

\subsection{Theoretical setup}
We begin by defining the geometric quantity we measure and then show how analyzing the rate alone does not determine it. The empirical study, intuition, and a toy-model confirmation follow in the subsequent subsections.

\paragraph{Per-sample posterior anisotropy.}
\label{sec:apost}

For a diagonal-Gaussian posterior ${q_\phi(z \mid x)} = \mathcal{N}(\mu_\phi(x), \mathrm{diag}(\sigma^2_\phi(x)))$, we summarize the per-sample variance allocation by the variance of the log-variances predicted by the encoder across latent coordinates:
\begin{equation}
A_{\mathrm{post}}(x) := \mathrm{Var}_{i \in \{1, \ldots, D\}} \big[\, \log \sigma^2_{\phi,i}(x) \,\big].
\label{eq:apost}
\end{equation}
High $A_{\mathrm{post}}(x)$ means a few coordinates carry most of the certainty while the rest remain noisy: the encoder has effectively concentrated information into a low-dimensional subspace, with the remaining coordinates acting as inactive units. Low $A_{\mathrm{post}}(x)$ means uncertainty is spread evenly across coordinates. We average $A_{\mathrm{post}}(x)$ over the dataset to obtain a per-model anisotropy score.

\paragraph{What determines this anisotropy.}
\label{sec:waterfilling}

The variance contribution to the rate is a scalar, and it is not informative about how variance is distributed across the coordinates of the latent representation. What determines this allocation? In the analytically tractable case of a Gaussian source under a quadratic distortion, classical results give a clean answer: the per-coordinate variances follow a water-filling profile:
\begin{equation}
\sigma^2_i \;\propto\; \frac{\beta}{\beta + 2 c_i}.
\label{eq:waterfilling}
\end{equation}
Here \(c_i\) is the distortion weight assigned to coordinate \(i\): large
\(c_i\) means that errors along that direction are expensive, while small
\(c_i\) means they are relatively cheap. Thus, directions that matter more under
the distortion receive lower posterior variance, i.e. more information. The
shape of the variance profile is therefore determined by the directional
structure of the distortion. For neural losses combined with deep encoders,
these effective weights are not analytically tractable, so how such losses
alter posterior anisotropy is an empirical question.

\subsection{Neural losses induce per-sample posterior isotropy}
\label{sec:observation}

To understand how neural losses influence the allocation of information in the latent space, we compare models trained with different $\lambda$ at \emph{matched rate}. We use the GECO algorithm \cite{DBLP:journals/corr/abs-1810-00597} to dynamically adjust $\beta$ during training so that the achieved rate converges to a specified target $\hat{R}$, and we sweep $\lambda$ across the same values used in Section~\ref{subsec:claim1-empirical}. The full setup is presented in Appendix~\ref{sec:experimental-setup}.

% \begin{figure}[t]
\begin{wrapfigure}{r}{0.45\textwidth}
  \vspace{-0.8em}  
  \centering
  \includegraphics[width=0.45\columnwidth]{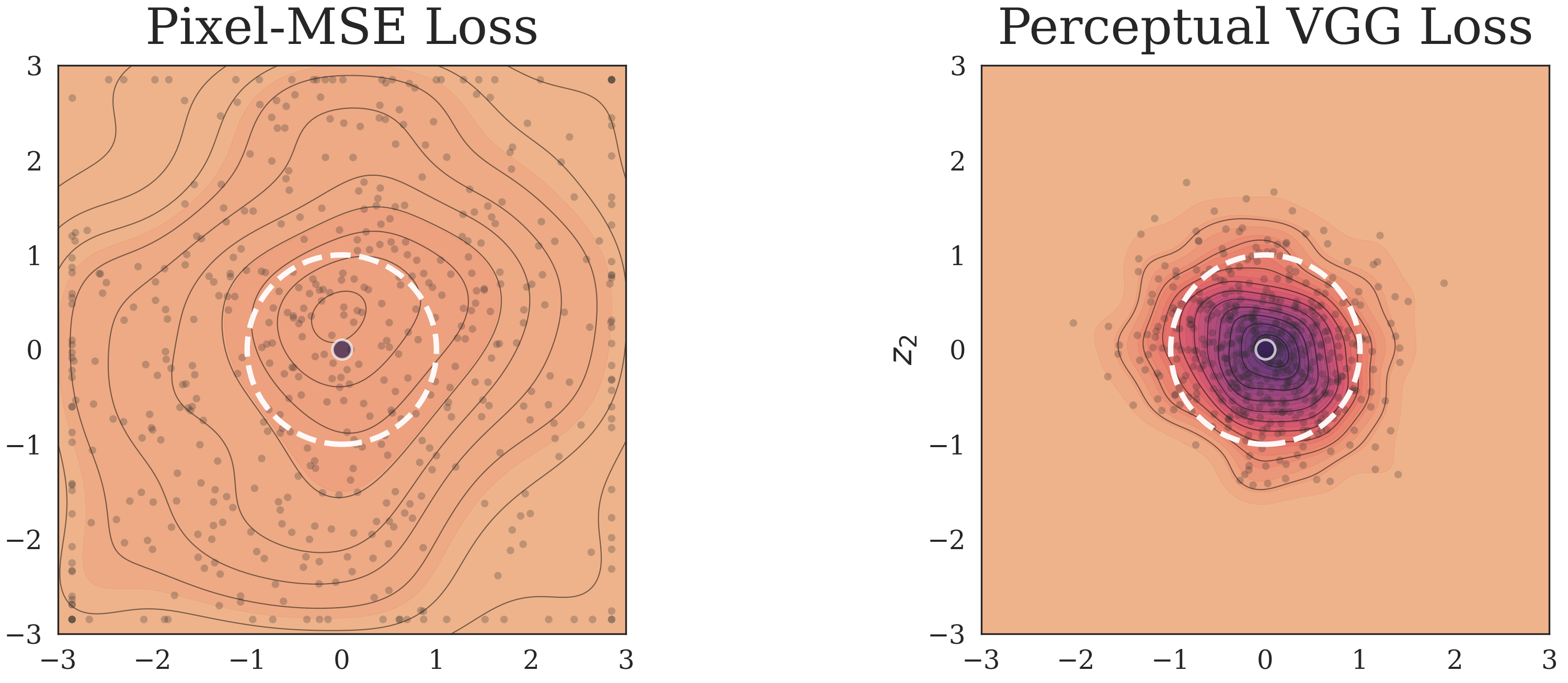}
  \caption{Equivalence classes under pixel vs.\ perceptual distortion. Dashed curves enclose equivalent perturbations under each loss. 
  \label{fig:equivalence-classes}
  }
% \end{figure}
  \vspace{-1em}
\end{wrapfigure}

We repeat the experiment across two dataset-model combinations, two neural losses (perceptual and discriminative), and 4 fixed target rate values. For each combination, we sweep $\lambda$ and measure the average per-sample posterior anisotropy $A_{\mathrm{post}}$.
Figure~\ref{fig:claim2-ratematched} reports the result for the pythae VAE trained on CelebA.

At every target rate, $A_{\mathrm{post}}$ decreases monotonically as $\lambda$ grows. The effect is small at the lowest target rate ($\hat{R} = 20$, blue line), where all posteriors are close to the prior and there is little anisotropy to redistribute, and large at higher rates ($\hat{R} \in \{60, 100, 140\}$). 

Figure~\ref{fig:ratematched-tinyimagenet} in Appendix~\ref{app:plots} shows that the same monotone decrease in $A_{\mathrm{post}}$ with $\lambda$ at matched rate holds for the AutoencoderKL trained on Tiny-ImageNet, indicating that the effect generalizes across architecture and dataset.

The relationship between ``isotropy of the loss in pixel space'' and ``isotropy of the resulting posterior'' runs opposite to the naive guess. Pixel SSE is the most isotropic loss in pixel space (every pixel is weighted equally) yet it produces the most anisotropic posteriors. Perceptual and discriminative losses are anisotropic in pixel space (they weight feature directions and patches non-uniformly) yet they produce the most isotropic posteriors.

\subsection{Intuition}
\label{sec:intuition}

\begin{wrapfigure}{r}{0.38\textwidth}
  \vspace{-1.5em}
% \begin{figure}
  \centering
  \includegraphics[width=0.36\textwidth]{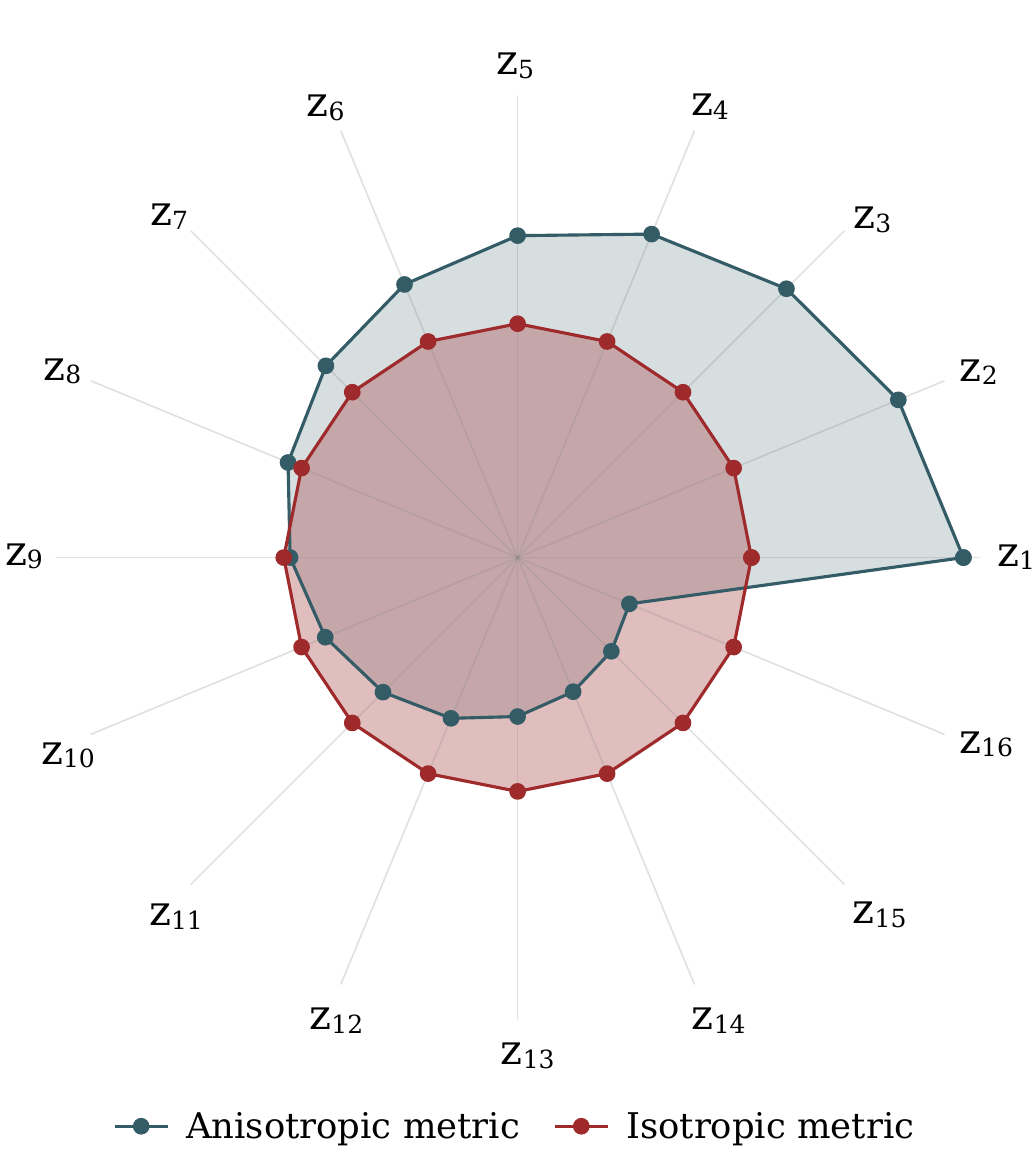}
  \caption{Toy model at matched rate: posterior standard deviation per latent dimension.}
  \label{fig:toy-model}
% \end{figure}
  \vspace{-1em}
\end{wrapfigure}

The mechanism behind this counter-intuitive direction can be read off the water-filling formula~\eqref{eq:waterfilling} combined with how neural losses act on pixel space.

Natural-image datasets are highly anisotropic in pixel space: a small number of principal components capture most of the dataset's variance. Since pixel SSE penalizes every pixel equally, the encoder is forced to spend rate preserving precisely those few high-variance directions. The other, lower-variance directions are necessarily left near the prior to respect the information budget.

Neural losses operate in a feature space that has already collapsed many pixel-level directions into perceptual equivalence classes (as shown in Figure~\ref{fig:equivalence-classes}). This removes the pressure to memorize specific patterns, so water-filling allocates rate more evenly across coordinates, producing a more isotropic variance profile.

This predicts the direction of the effect, but not its magnitude, which depends on the interaction between the encoder and the loss.
Each latent coordinate inherits an effective distortion weight that reflects both the directions the encoder uses to embed information in image space and how strongly the loss penalizes errors along those directions. A more complete mechanistic characterization would require quantifying these effective per-coordinate weights for trained deep encoders.
% A more complete mechanistic characterization would require studying $(W^\top M^\top M W)_{ii}$ for trained deep encoders, which we leave to future work.

\subsection{Toy model of anisotropy}
\label{sec:toy}

To confirm and visualize that the water-filling mechanism operates as described, we instantiate it in a closed-form linear-Gaussian setting.

We fix dimension $D = 16$, a linear decoder $W = I_D$, and evaluate reconstructions under a quadratic metric $\|M(x - Wz)\|^2$. We compare two choices of $M$: an isotropic $M_{\mathrm{iso}} = I_D$, and an anisotropic one, $M_{\mathrm{aniso}}$ with exponentially ramping weights. For any target rate, we can pick $\beta$ separately for each metric so that both models converge to the same variance-KL. Full derivations are in Appendix~\ref{app:toymodel}.

Figure~\ref{fig:toy-model} shows the per-dimension posterior standard deviation at a matched rate. The isotropic metric traces a near-perfect circle: every coordinate carries the same uncertainty. The anisotropic metric is visibly lopsided, with certainty concentrated in a subset of coordinates. Both profiles correspond to the same scalar rate, showing how distortion geometry alone can fully reshape the variance profile.

\section{Related Work}
\label{sec:related-work}

\paragraph{Rate-distortion.}

\cite{cover} presents rate-distortion theory as the framework that, given a source distribution and a distortion measure, characterizes the minimum coding rate needed to achieve a specified expected distortion. VAEs \cite{DBLP:journals/corr/KingmaW13} are widely studied as approximate solutions to rate-distortion problems. The $\beta$-VAE framework \cite{higgins2017betavae} first demonstrated that scaling the KL regularizer manages the tradeoff between reconstruction quality and latent capacity, where higher $\beta$ values promote factorization at the expense of fine detail. 
This concept was later refined through capacity control, which uses targeted KL levels and gradual capacity increases to mitigate posterior collapse \cite{burgess2018understandingdisentanglingbetavae}. 

A more formal RD viewpoint characterizes the VAE objective as a Lagrangian relaxation, revealing that models with powerful decoders frequently under-utilize their latent space \cite{DBLP:conf/icml/AlemiPFDS018}. While multi-rate variants allow a single model to span an entire RD frontier \cite{bae2023multiratevaetrainonce}, existing research typically focuses on navigating these curves by adjusting $\beta$. These works generally treat reconstruction loss as a fixed parameter, leaving the impact of different distortion choices on the attainable RD region largely unexamined.

\paragraph{Perceptual and adversarial tradeoffs.}
Changing the distortion functional has been extensively explored in practice, but usually without an explicit RD analysis. Perceptual VAEs \cite{johnson2016perceptuallossesrealtimestyle} replace pixel-wise reconstruction terms with feature-space distances, for instance penalizing discrepancies in VGG activations or discriminator features~\cite{DBLP:conf/icml/LarsenSLW16,7926714}, which leads to sharper and more semantically meaningful reconstructions \cite{7926714}. In learned image compression, AE-based codecs combine MSE, perceptual distances like LPIPS \cite{zhang2018unreasonableeffectivenessdeepfeatures}, and adversarial losses to improve subjective quality at a given bitrate \cite{mentzer2020highfidelitygenerativeimagecompression}. 

From a theoretical standpoint, the inevitable tradeoff between distortion and perception has been rigorously characterized \cite{8578750}, and extended to a rate-distortion-perception framework \cite{blau2019rethinkinglossycompressionratedistortionperception}. 
These works show that enforcing perceptual quality constraints changes the achievable RD frontier: improved perceptual quality generally requires sacrificing either distortion or rate. However, this theory does not directly address the learned-rate setting considered here, where the bitrate is selected implicitly by optimizing a KL-regularized objective. 
% We revisit this distinction in Section~\ref{sec:your_section_label}, and show that changing the reconstruction loss can alter the learned KL rate itself.
% These results suggest that changing the distortion metric or adding perceptual constraints moves the RD curve, but existing VAE-based models do not quantify how this change manifests in terms of KL rate, or latent uncertainty.

\paragraph{Latent geometry.}
Research on latent space geometry explores how information constraints influence representation structure. Disentanglement methods like $\beta$-TCVAE \cite{chen2019isolatingsourcesdisentanglementvariational} and DIP-VAE \cite{kumar2018variational} modify the KL regularizer to penalize correlation, encouraging statistically independent dimensions. Theoretical analyses suggest $\beta$-VAE objectives are rotationally invariant, primarily controlling posterior overlap rather than axis alignment with generative factors \cite{mathieu2019disentanglingdisentanglementvariationalautoencoders}. Empirically, varying $\beta$ leads to "active units" and concentrated variance in specific dimensions. 

Closest to our work, recent studies \cite{10.1371/journal.pcbi.1012952} have connected RD tradeoffs to geometric effects (prototypization, specialization, orthogonalization) in synthetic datasets, though these typically employ fixed pixel-level distortions and do not study how different distortion metrics affect the optimal rate or the anisotropy of the learned posterior distributions. Also related, \cite{chen2017variationallossyautoencoder} combines VAEs with autoregressive decoders to bias the latent code toward global structure while leaving fine texture to the decoder.

Finally, VAE-style models now serve as backbones in large-scale generative systems, such as latent diffusion models \cite{rombach2022highresolutionimagesynthesislatent}. These systems utilize mixtures of pointwise, perceptual, and adversarial losses to compress data into low-dimensional spaces. While they implicitly explore rate-distortion-perception tradeoffs in high-dimensional, realistic domains, the choice of loss functions is often treated as an engineering design rather than a subject of theoretical study.

% TODO(tkol): Write this section.
% Suggested structure:
% - Rate--distortion / information bottleneck views of VAEs
% - Perceptual + adversarial reconstruction losses for VAEs
% - Posterior collapse / KL annealing / free bits / beta scheduling
% - Latent geometry, anisotropy, and downstream probing for VAEs
% - Spectral structure and spatial latents (if you keep that claim)

\section{Discussion}
\label{sec:discussion}

\subsection{Latent shape design}

The latent space geometry of VAEs has immense consequences on downstream tasks like diffusion modeling. However, current practice for training VAEs is mostly empirical: most modern VAEs are trained with a combination of pixel SSE, perceptual, and adversarial losses, despite their complexity and lack of theoretical grounding. 

Improvements in reconstruction quality alone are not a sufficient justification for the effectiveness of a loss in creating latents that are suitable for diffusion modeling: in fact, reconstruction FID and generation FID are often at odds with each other \cite{yao2025reconstruction}. This combination of neural training losses must therefore be doing something useful in \emph{latent} space, not only in data space.

Dieleman \cite{dieleman2025latents} decomposes latent-space design into three roughly orthogonal axes: \emph{capacity} (how many bits the latents carry), \emph{curation} (which bits from the input are retained), and \emph{shape} (how this information is arranged across coordinates). Curation has received recent attention through representation-alignment methods \cite{yao2025reconstruction, chen2025masked, yu2024representation}; capacity and shape have so far been controlled mainly through $\beta$ and through auxiliary regularizers \cite{kouzelis2025eqvaeequivarianceregularizedlatent, skorokhodov2025improvingdiffusabilityautoencoders, leng2025repa}. 

In this paper we show that the commonly used perceptual and adversarial losses are themselves direct controllers of both capacity and shape. Section~\ref{sec:claim1} shows that such neural objectives lower the achieved rate, making low capacity affordable without reconstruction collapse. Section~\ref{sec:claim2} shows that they redistribute the available information uniformly across coordinates rather than concentrating it in a few active dimensions, directly shaping the latent geometry.

This addresses a critical gap in current literature: to train optimal VAEs we must first understand what our choice of objectives implies for the latent representations. 

\subsection{Extending the rate-distortion-perception tradeoff}
\label{subsec:discussion-rd}
Blau and Michaeli \cite{blau2019rethinkinglossycompressionratedistortionperception} study the rate–distortion–perception tradeoff by asking how the classical RD frontier changes when perceptual quality is constrained. In their experiments, rate is fixed externally through the latent dimensionality and quantization levels, and models are trained at each prescribed rate with different distortion/perception weights; they show that enforcing an external constraint on perceptual quality elevates the RD curve, requiring higher pixel distortion and/or higher rate. 

Our work considers the practical learned-rate regime, where the rate is not fixed but emerges from optimizing a Lagrangian objective containing a KL divergence and a reconstruction loss. We show that, when the reconstruction loss is changed from pure pixel SSE to a  mixture of SSE and perceptual/GAN terms, the optimum shifts to lower rates, revealing an effect that is not captured by fixed-rate RDP analysis.

\section{Limitations and future work}
\label{sec:limits}
Our theory is about ideal Shannon rate-distortion curves, but our experiments measure the KL that finite VAEs reach when trained with SGD. The VAE uses an amortized encoder instead of the optimal per-input kernel, its KL is an upper bound on mutual information rather than mutual information itself, and ELBO training is non-convex, so the achieved rate can also depend on optimization dynamics that Shannon theory does not capture. On the geometry side, the water-filling argument behind our anisotropy analysis is exact only for Gaussian sources with quadratic distortion and linear decoders. For deep encoders trained with neural losses, it predicts the direction of the effect but not its magnitude, which is left to experimental evidence.

\paragraph{Future work.} 
A natural next step is to study in more depth the mechanisms by which neural losses shape the geometry of the latents, and to design losses that explicitly induce desirable properties in the latent space. We believe that this would be an extremely impactful extension of this work. A second direction is to study how the geometric property we identify affects the ability to train diffusion models on top of these latents. This would clarify whether and when the shape changes induced by neural losses are actually beneficial for downstream generation.

\bibliographystyle{plainnat}
\bibliography{bibliography}

%%%%%%%%%%%%%%%%%%%%%%%%%%%%%%%%%%%%%%%%%%%%%%%%%%%%%%%%%%%%

\appendix

\crefname{subsection}{Appendix}{Appendices}
\Crefname{subsection}{Appendix}{Appendices}

\newpage
\raggedbottom

\section{Additional plots}
\label{app:plots}

\begin{figure}[h]
  \centering
  \includegraphics[width=0.49\columnwidth]{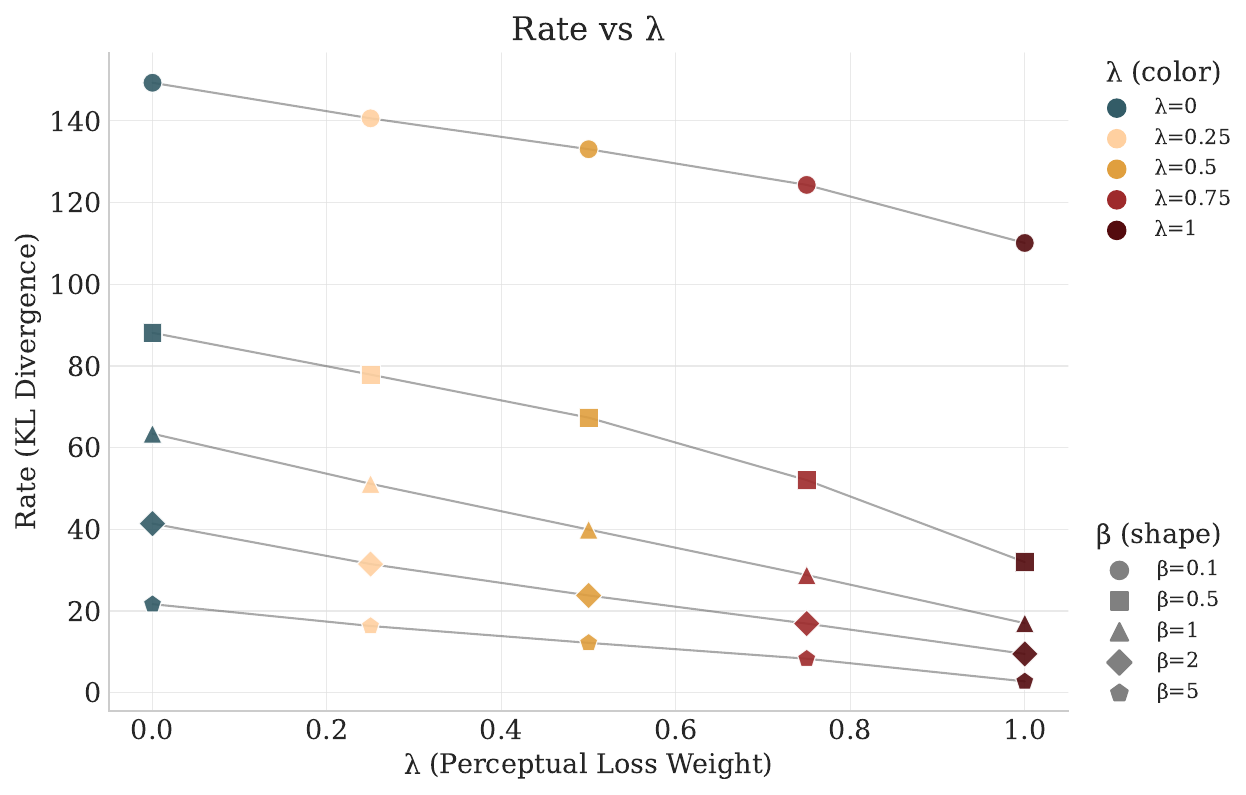}
  \hfill
  \includegraphics[width=0.49\columnwidth]{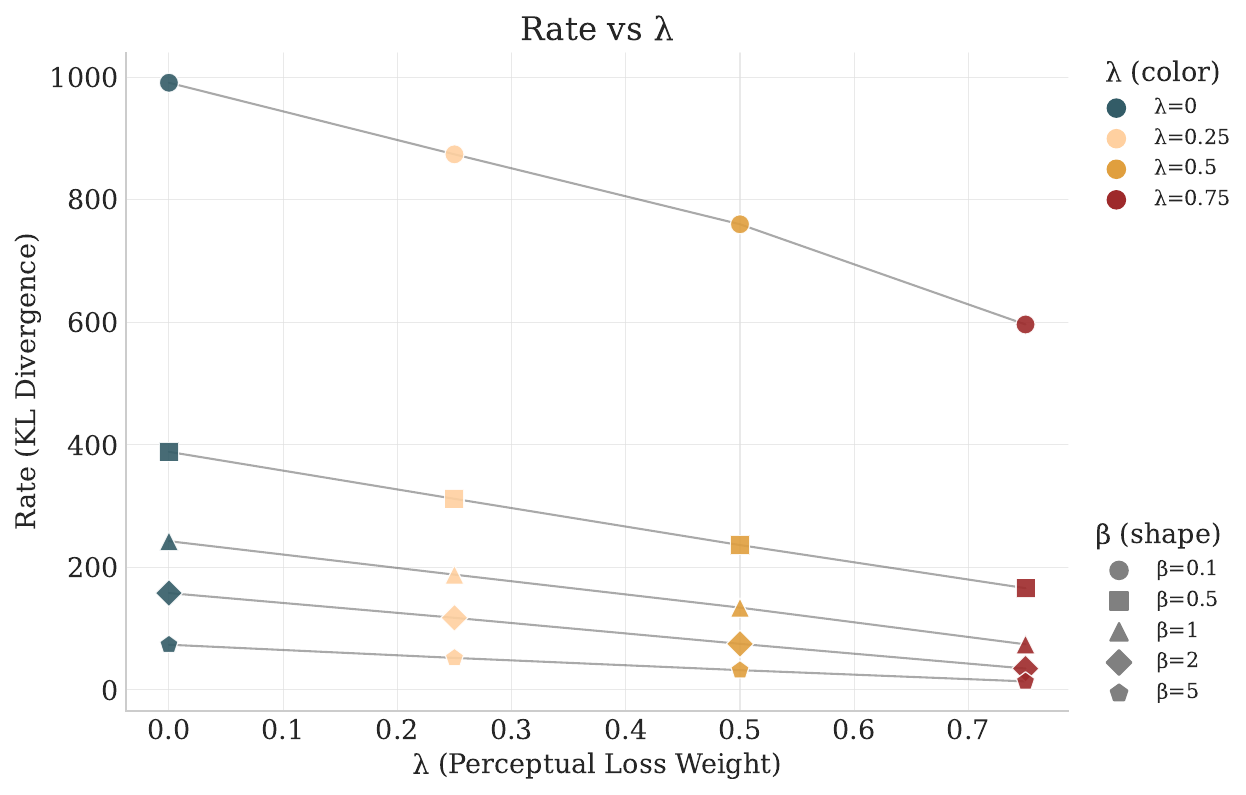}
  \caption{Experimental results supporting the claim in Section~\ref{subsec:claim1-empirical} for the AutoencoderKL architecture, trained on CelebA (left) and Tiny-ImageNet (right). It shows (for each $\beta$ value) a monotonically decreasing relation between $\lambda$, the weight of the perceptual loss, and the rate of the model at convergence.}
  \label{fig:autoencoder_kl_claim1}
\end{figure}

\begin{figure}[h]
  \centering
  \includegraphics[width=0.49\columnwidth]{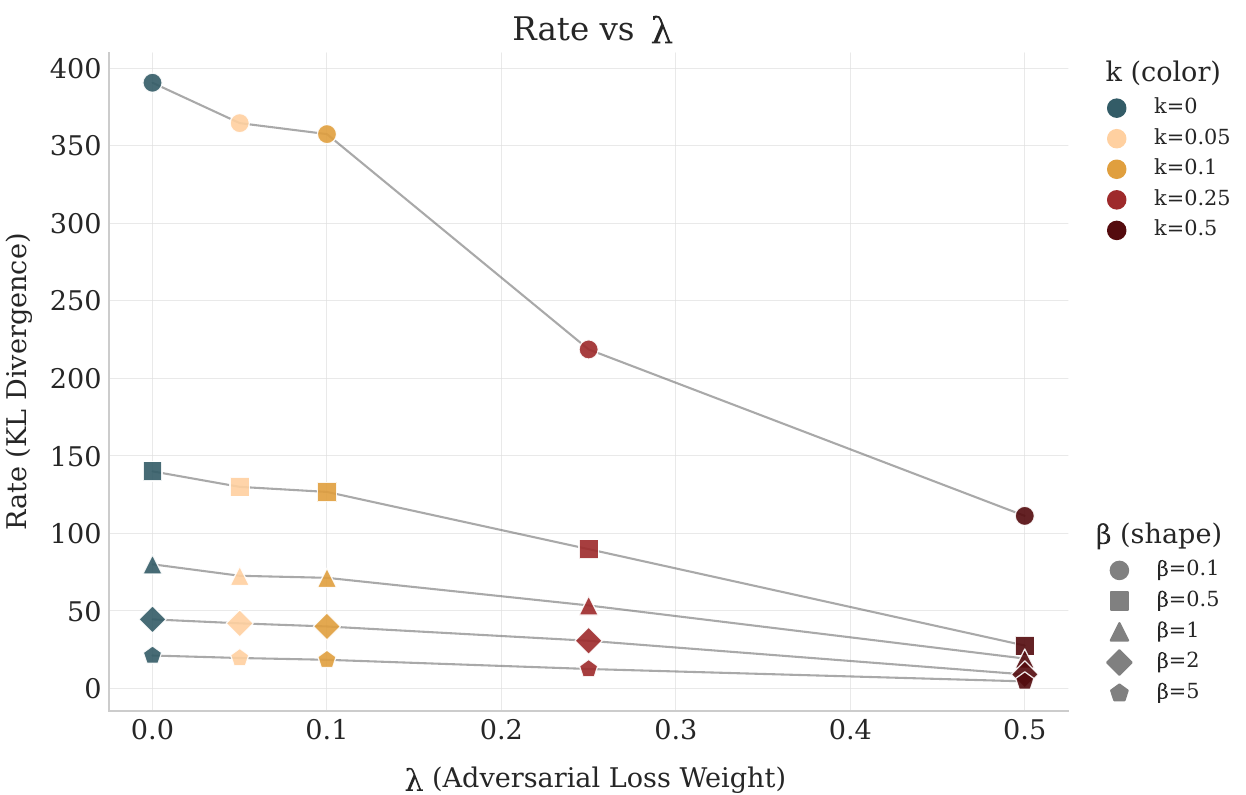}
  \caption{Experimental results supporting the claim in Section~\ref{sec:claim1} for the AutoencoderKL architecture trained on Tiny-ImageNet. It shows (for each $\beta$ value) a monotonically decreasing relation between $\lambda$, the weight of the adversarial loss, and the rate of the model at convergence.}
  \label{fig:autoencoder_kl_claim1-k}
\end{figure}

\begin{figure}[h]
  \centering
  \includegraphics[width=0.49\columnwidth]{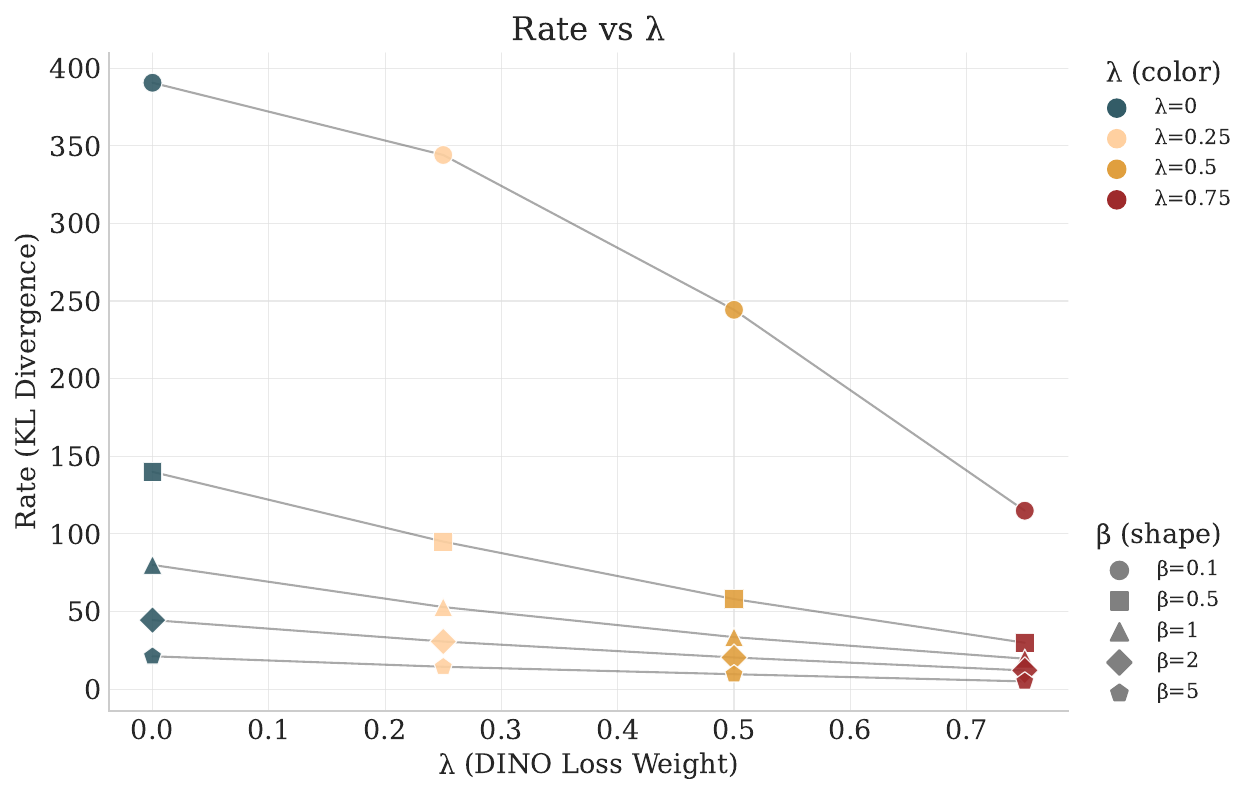}
  \caption{Experimental results supporting the claim in Section~\ref{sec:claim1} for the AutoencoderKL architecture trained on Tiny-ImageNet. It shows (for each $\beta$ value) a monotonically decreasing relation between $\lambda$, the weight of the DINO loss, and the rate of the model at convergence.}
  \label{fig:autoencoder_kl_claim1-dino}
\end{figure}

\begin{figure}[h]
  \centering
  \includegraphics[width=0.49\columnwidth]{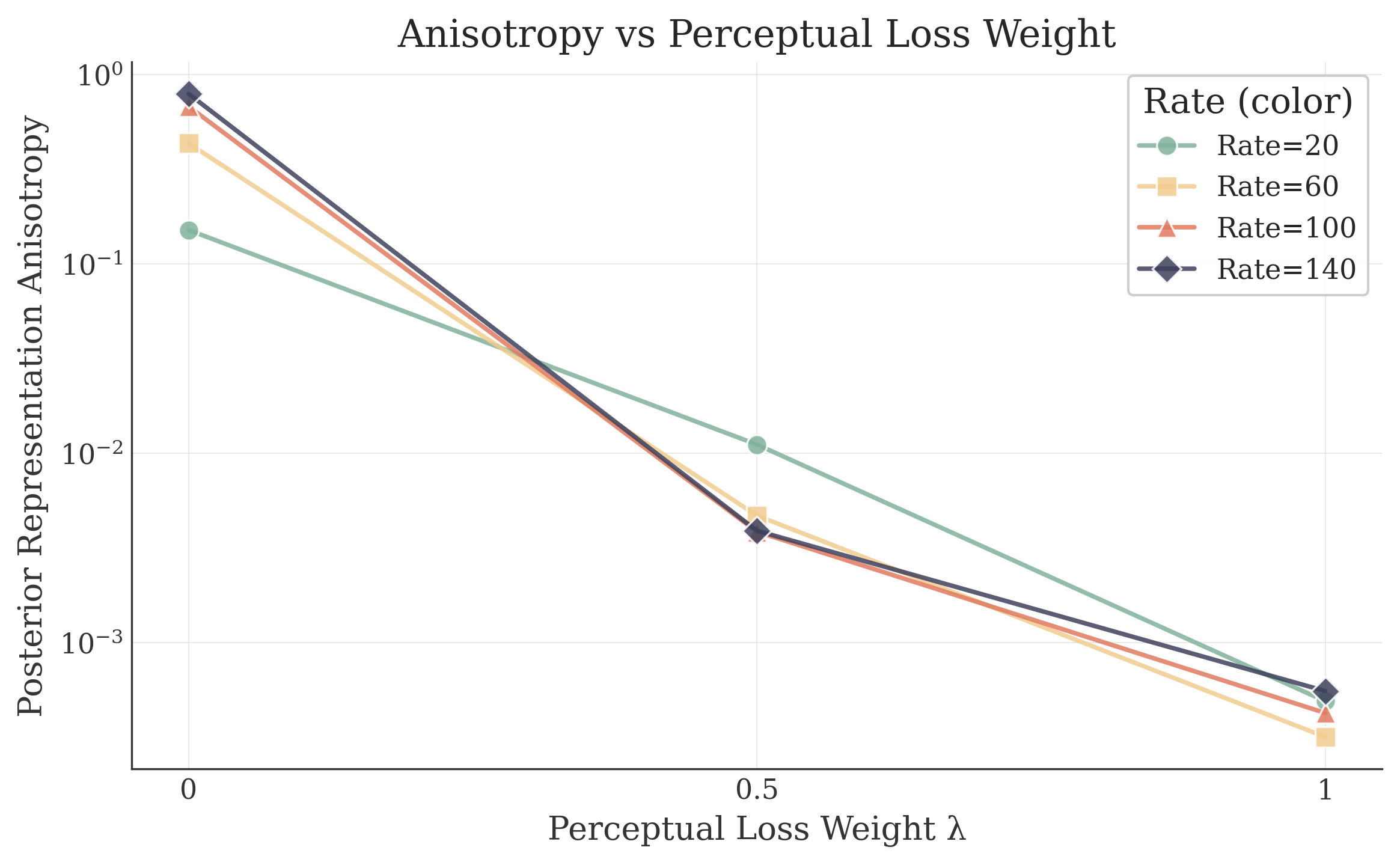}
  \caption{Rate-matched training of the AutoencoderKL architecture on Tiny-ImageNet using the perceptual loss LPIPS. Shows how, for each fixed rate (color), the average per-sample posterior representations become less anisotropic as the neural contribution to the reconstruction loss grows.
  \label{fig:ratematched-tinyimagenet}
  }
\end{figure}

\begin{figure}[h]
  \centering
  \includegraphics[width=0.7\columnwidth]{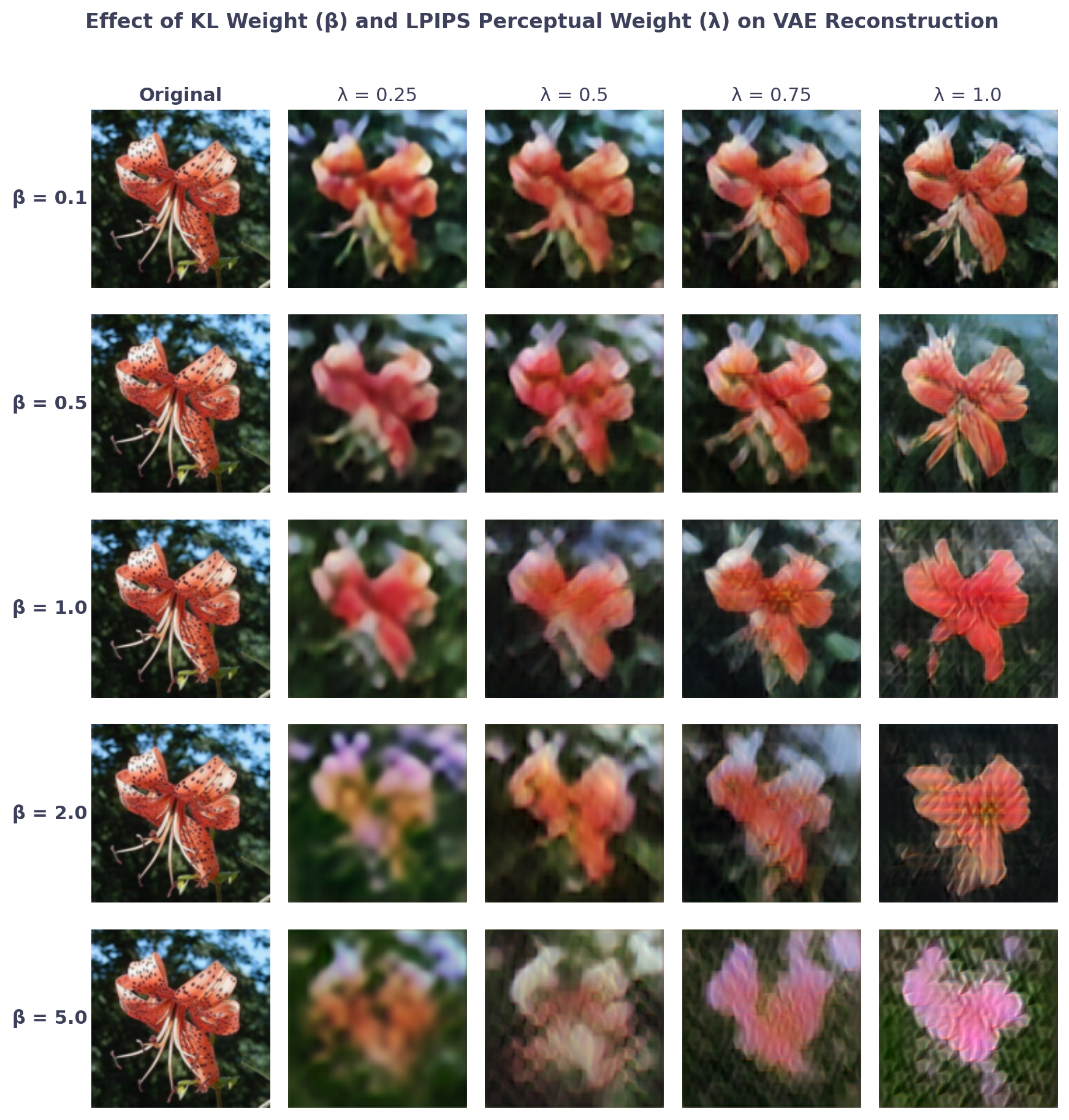}
  \caption{Reconstructions of the same input as the weight $\lambda$ of the perceptual loss and the coefficient $\beta$ vary. From left to right: reconstructions with an increasing perceptual term ($\lambda$). Increasing $\lambda$ progressively sharpens details and better preserves identity compared to the SSE-only baseline, while keeping the underlying latent regularization unchanged.}

  \label{fig:reconstructions_bis}
\end{figure}

\FloatBarrier
\clearpage
\section{Experimental setup}
\label{sec:experimental-setup}

This appendix documents the complete experimental protocol used throughout the
paper. All experiments share the same training stack and optimization schedule; the only quantities that vary across studies are the
reconstruction term $\widehat{D}$, its mixing weight $\lambda$, and either the
KL weight $\beta$ or the rate target $\hat{R}$ (when GECO is enabled).

\paragraph{Architectures and datasets.}
We evaluate two complementary VAE families to ensure that our conclusions are
not artifacts of a particular decoder inductive bias: a \emph{vector-latent}
VAE built with \texttt{pythae}~\cite{chadebec2022pythae} (\texttt{VAE} class,
$\sim$25.2M parameters), representative of the classical formulation, and a
\emph{spatial-latent} VAE based on the \href{https://huggingface.co/docs/diffusers/main/en/api/models/autoencoderkl#diffusers.AutoencoderKL}{AutoencoderKL} module from the
\texttt{diffusers} library ($\sim$14.4M parameters) ,
which preserves spatial latent grids and matches the encoder--decoder
backbone used in modern latent diffusion models. 

We use two image datasets:
\textbf{CelebA}\cite{liu2015faceattributes}, providing aligned, semantically
homogeneous human faces (the canonical benchmark for VAE rate--distortion
analysis), and \textbf{Tiny~ImageNet} \cite{deng2009imagenet}, providing a heterogeneous,
multi-class natural image distribution with strong intra-class variability.
The two datasets cover the low-entropy/structured and the
high-entropy/diverse extremes that bracket the use cases of generative
autoencoders. Images are center-cropped and normalized to $[-1,1]$, with horizontal-flip augmentation at training time.
The architectural configuration of the spatial-latent backbone is reported
in Table~\ref{tab:arch}.

\begin{table}[h]
\centering
\small
\caption{Architecture configuration of the spatial-latent VAE
(\texttt{AutoencoderKL}). The vector-latent backbone uses the default
\texttt{pythae} \texttt{VAE} configuration.}
\label{tab:arch}
\begin{tabular}{ll}
\toprule
\textbf{Field} & \textbf{Value} \\
\midrule
\texttt{latent\_channels}            & $4$ \\
\texttt{block\_out\_channels}        & $[128, 256, 512, 512]$ \\
\texttt{layers\_per\_block}          & $2$ \\
\texttt{act\_fn}                     & SiLU \\
\texttt{norm\_num\_groups}           & $32$ \\
\texttt{scaling\_factor}             & $1.0$ \\
\texttt{force\_upcast}               & true \\
\texttt{use\_quant\_conv}            & true \\
\texttt{use\_post\_quant\_conv}      & true \\
\texttt{mid\_block\_add\_attention}  & true \\
\bottomrule
\end{tabular}
\end{table}

\paragraph{Reconstruction terms $\widehat{D}$ and training objective.}
We evaluate four reconstruction objectives: (i)~pixel-space squared error (SSE);
(ii)~a \emph{perceptual} VGG-feature distance, in the
LPIPS parameterization \citep{zhang2018unreasonableeffectivenessdeepfeatures}; 
(iii)~a \emph{perceptual} distance based on the self-supervised DINOv2 backbone \cite{oquab2023dinov2};
(iv)~a \emph{discriminative} PatchGAN hinge loss with feature
matching\citep{isola2018imagetoimagetranslationconditionaladversarial}. 

The training objective is
\[
\mathcal{L} \;=\; \mathcal{B}\!\left(\lambda\,\widehat{D} + (1-\lambda)\,\text{SSE}\right)
              \;+\; \beta\,\text{KL},
\]
swept over $\lambda\in\{0,0.25,0.5,0.75,1\}$ and $\beta\in\{0.1,0.5,1,2,5\}$.

\paragraph{Loss balancer $\mathcal{B}$.}
Reconstruction terms operate at very different natural scales (pixel SSE is
$O(\!10^{-2}\!)$, VGG/DINO features are $O(\!1\!)$, GAN losses fluctuate by
orders of magnitude during training), so fixed scalar weights are not
comparable across $\widehat{D}$ choices. We adopt the EnCodec
balancer~\cite{defossez2022high}, which decouples the weight of each
loss from its magnitude. Given losses $\ell_i$ and gradients
$g_i = \partial \ell_i / \partial \hat{x}$ at the decoder output, the
balancer rescales each gradient as
$\tilde{g}_i = R\cdot \frac{\lambda_i}{\sum_j \lambda_j} \cdot
\frac{g_i}{\langle \lVert g_i \rVert_2 \rangle}$,
with $\langle \cdot \rangle$ an EMA of the per-loss gradient norm, and
backpropagates $\sum_i \tilde{g}_i$ through the network. With
$\sum_i \lambda_i = 1$ each $\lambda_i$ is literally the fraction of the
gradient supplied by $\ell_i$, making sweeps over $\lambda$ directly
interpretable across all four reconstruction terms. The KL term is added
\emph{outside} the balancer, in line with EnCodec's treatment of the
commitment loss.

\paragraph{Dynamic $\beta$ via GECO.}
For the constant-rate experiments (Section~\ref{sec:claim2}) we replace
the fixed $\beta$ with the dynamic Lagrange multiplier of
GECO \cite{DBLP:journals/corr/abs-1810-00597}, which formulates training as the constrained
problem $\min \widehat{D}$ s.t.\ $\mathbb{E}[\text{KL}] \le \hat{R}$. We
target the rate directly via the relative violation
$\mathcal{C}_t = \big(\text{KL}^{\text{ma}}_t - \hat{R}\big)/\hat{R}$,
where $\text{KL}^{\text{ma}}_t$ is an EMA of the per-step KL. The multiplier
is updated multiplicatively,
$\beta_t = \beta_{t-1}\,(1 + \alpha\,\mathcal{C}_t)$,
with a symmetric clamp keeping $\beta_t$ in a safe range. The multiplicative
form keeps $\beta>0$ without explicit projection, the EMA absorbs minibatch
noise, and to decouple GECO from the warm-up transient the multiplier is
held at $1$ during the warm-up phase, with only the EMA tracked. We use
rate targets $\hat{R}\in\{20,60,100,140\}$.

\paragraph{Optimization and hardware.}
All numerical settings of the training loop, of the balancer, and of GECO
are summarized in Table~\ref{tab:hparams}. Training is run on
$2{\times}$~NVIDIA A100~64\,GB GPUs with PyTorch~Lightning DDP. Generator
and discriminator (when active) are optimized by separate Adam instances:
the generator's learning rate follows a cosine schedule over the full
training horizon, while the discriminator's is held constant at half the
generator's rate, with momentum $(\beta_1,\beta_2)=(0,0.9)$ as is standard
in GAN training. 

\begin{table}[h]
\centering
\small
\caption{Optimization, balancer, and GECO hyperparameters. All values are
fixed across runs; only $\lambda$, $\beta$ (or $\hat{R}$ in rate-matched
experiments), and $\widehat{D}$ vary across the sweeps.}
\label{tab:hparams}
\begin{tabular}{lll}
\toprule
\textbf{Group} & \textbf{Field} & \textbf{Value} \\
\midrule
\multirow{6}{*}{Optimization}
  & Optimizer (generator)        & Adam, $(\beta_1,\beta_2)=(0.9,0.999)$, no weight decay \\
  & Learning rate (generator)    & $10^{-4}$, cosine annealing over \texttt{max\_epochs} \\
  & Optimizer (discriminator)    & Adam, $(\beta_1,\beta_2)=(0,0.9)$ \\
  & Learning rate (discriminator)& $5\!\times\!10^{-5}$, constant \\
  & Precision                    & bf16 mixed \\
  & Batch size / device          & $24$ \\
\midrule
\multirow{4}{*}{Schedule}
  & Image resolution             & $128\!\times\!128$ \\
  & Max epochs                   & $25$ \\
  & \texttt{val\_check\_interval}& $100$ steps \\
  & \texttt{best\_checkpoint\_k} & $3$ \\
\midrule
\multirow{2}{*}{KL warm-up}
  & Schedule                     & sine, $0\!\to\!1$ \\
  & Steps                        & $1000$ \\
\midrule
\multirow{2}{*}{Balancer}
  & Reference norm $R$           & $1$ \\
  & EMA decay                    & $0.999$ \\
\midrule
\multirow{4}{*}{GECO}
  & Step size $\alpha$           & $5\!\times\!10^{-4}$ \\
  & EMA decay (KL)               & $0.99$ \\
  & Initial $\beta_0$            & $1$ \\
  & Clamp                        & $[10^{-4},\,10^{4}]$ \\
\midrule
Hardware
  & GPUs                         & $2 \times$ A100 64\,GB, DDP \\
\bottomrule
\end{tabular}
\end{table}

% We verify empirically all of our claims by training VAE models on the CelebA dataset \cite{liu2015faceattributes} and Tiny-Imagenet \cite{deng2009imagenet}. We repeat our experiments throughout two different model architectures:

% \begin{enumerate}[label=(\roman*), leftmargin=20pt, topsep=0pt]
%     \item the traditional VAE architecture as implemented by \texttt{pythae} \cite{chadebec2022pythae} (25.2M params);
%     \item a spatial-feature VAE as proposed in \href{https://huggingface.co/docs/diffusers/main/en/api/models/autoencoderkl#diffusers.AutoencoderKL}{AutoencoderKL} from the Diffusers library (14.4M params).
% \end{enumerate} 
\FloatBarrier
%\input{sections/appendix0}
%\FloatBarrier
\clearpage
\section{Weaker distortions induce no larger optimal rate}
\label{app:claim1-complete-proofs}

Here we provide formal proof of our first claim stated in Section~\ref{sec:claim1}. The proof will be divided in the following sections:

\begin{itemize}
    \item \textit{Preliminaries and Definitions} We formalize a notion of \emph{weakness} between distortion functions via pointwise domination (Definitions~\ref{def:domination} and~\ref{def:affine-domination})
    \item \textit{Weaker distortion implies no larger optimal rate} We show that a weaker distortion has an RD curve that is nowhere larger (\Cref{thm:rd-domination,thm:rd-affine-domination}).
    The key idea is set inclusion: at a fixed distortion budget, the feasible set of reconstruction kernels for a weaker distortion
    contains the feasible set of a stronger one; hence the infimum mutual information cannot increase.
    \item \textit{When the Gap is Strict}
    Because set inclusion can be tight, the RD ordering need not be     strict. We give a simple sufficient condition
    for a strict gap: if the weaker distortion ignores a component of the   data that the stronger distortion must reconstruct below
    its zero-rate baseline, then any feasible reconstruction for the    stronger distortion must transmit \emph{additional} information.
    \item \textit{Application: Proof of Weakness for Common Distortions}
    Finally, we show that common perceptual and discriminator-based objectives used in VAE training fit the framework:
    (i) feature-space MSEs are dominated by pixel MSE under a Lipschitz assumption on the feature map; and
    (ii) both feature matching and (shifted) hinge generator losses admit domination (or affine domination) bounds by pixel MSE
    under standard regularity assumptions.
\end{itemize}

\subsection{Preliminaries and Definitions}
\label{app:framework:rd}
We briefly recall standard notions from Shannon rate-distortion theory \cite{cover} that we will use to state and compare optimal information-distortion tradeoffs. In particular, we formalize (i) reconstruction rules as conditional kernels, (ii) distortion budgets via an expected loss, and (iii) the resulting Shannon RD function $R_d(\Delta)$.

\begin{definition}[Reconstruction kernel]
  \label{def:recon-kernel}
  Let $X\sim p_X$ take values in a measurable space $\mathcal X$,
  and let $\hat X$ take values in a measurable space $\hat{\mathcal X}$.
  A \emph{reconstruction kernel} is a conditional distribution $r(\hat x\mid x)$
  inducing the joint $p_X(x)\,r(\hat x\mid x)$ and hence the pair $(X,\hat X)$.
\end{definition}

\begin{definition}[Distortion, achieved distortion, and rate]
  \label{def:distortion-rate}
  A \emph{distortion} is a measurable function $d:\mathcal X\times\hat{\mathcal X}\to[0,\infty)$.
  For a reconstruction kernel $r(\hat x\mid x)$, define the achieved distortion and rate by
  \begin{equation}
    \label{eq:rd-achieved}
    \mathcal D(r;d) := \mathbb E[d(X,\hat X)],
    \qquad
    \mathcal R(r) := I(X;\hat X),
  \end{equation}
  where the expectation and mutual information are computed under the joint distribution
  $p_X(x)\,r(\hat x\mid x)$.
\end{definition}

\begin{definition}[Shannon rate-distortion function]
  \label{def:shannon-rd}
  For a distortion budget $\Delta\ge 0$, the \emph{Shannon rate-distortion function} is
  \begin{equation}
    \label{eq:shannon-rd}
    R_d(\Delta)
    :=
    \inf_{r:\; \mathcal D(r;d)\le \Delta}\; \mathcal R(r).
  \end{equation}
\end{definition}

Thereafter, we introduce explicitly two important remarks to contextualize these definitions:

\begin{remark}[Interpretation of $R_d(\Delta)$]
  The quantity $R_d(\Delta)$ is the smallest mutual information $I(X;\hat X)$
  achievable by any (possibly stochastic) reconstruction rule $\hat X\sim r(\cdot\mid X)$
  while keeping the expected distortion $\mathbb E[d(X,\hat X)]$ at most~$\Delta$.
\end{remark}

\begin{remark}[Connection to encoder-decoder models]
  In a VAE (or any stochastic autoencoder), sampling $Z\sim q_\phi(z\mid x)$ and then
  sampling $\hat X\sim p_\theta(x\mid Z)$ induces a reconstruction kernel
  \begin{equation}
    \label{eq:vae-induced-kernel}
    r_{\phi,\theta}(\hat x\mid x)
    :=
    \int p_\theta(\hat x\mid z)\,q_\phi(z\mid x)\,dz.
  \end{equation}
  Our claim reasons about how changing the distortion $d$ reshapes the RD curve \eqref{eq:shannon-rd},
  and we use this Shannon-level ordering as a clean theoretical proxy for what is observed under
  ELBO training.
\end{remark}

To formalize comparisons between distortion functions, we next introduce simple pointwise domination relations.
These relations provide a convenient sufficient condition for one distortion constraint to imply another, and will
translate directly into monotonicity statements for the corresponding Shannon RD functions.

\begin{definition}[Domination up to scale]
  \label{def:domination}
  For $c>0$, write $d_2 \preceq_c d_1$ if
  \begin{equation}
    \label{eq:domination}
    d_2(x,\hat x)\le c\,d_1(x,\hat x)\qquad \forall (x,\hat x)\in\mathcal{X}\times\hat{\mathcal{X}}.
  \end{equation}
\end{definition}

\begin{definition}[Affine domination]
  \label{def:affine-domination}
  For $c>0$ and $b\ge 0$, write $d_2 \preceq_{c,b} d_1$ if
  \begin{equation}
    \label{eq:affine-domination}
    d_2(x,\hat x)\le c\,d_1(x,\hat x)+b\qquad \forall (x,\hat x)\in\mathcal{X}\times\hat{\mathcal{X}}.
  \end{equation}
\end{definition}

\subsection{Weaker Distortion Implies No Larger Optimal Rate}
\label{app:claim1:core}

The following results are the basic monotonicity principles used throughout the paper.

% \begin{theorem}[RD ordering under affine domination]
%   \label{thm:rd-affine-domination}
%   If $d_2 \preceq_{c,b} d_1$ for some $c>0$ and $b\ge 0$, then for every $\Delta\ge b$,
%   \begin{equation}
%     R_{d_2}(\Delta)
%     \le
%     R_{d_1}\!\left(\frac{\Delta-b}{c}\right).
%   \end{equation}
% \end{theorem}

% \begin{proof}
% For any reconstruction kernel $r$,
% \[
%     \mathcal D(r;d_2)
%     \le c\,\mathcal D(r;d_1)+b .
% \]
% Therefore,
% \[
%     \{r:\mathcal D(r;d_1)\le(\Delta-b)/c\}
%     \subseteq
%     \{r:\mathcal D(r;d_2)\le\Delta\}.
% \]
% Taking the infimum of $I(X;\hat X)$ over the larger feasible set cannot increase the value, yielding the claim.
% \end{proof}

% \begin{corollary}[RD ordering under domination]
%   \label{thm:rd-domination}
%   If $d_2 \preceq_c d_1$, then for every $\Delta\ge 0$,
%   \begin{equation}
%     R_{d_2}(\Delta)
%     \le
%     R_{d_1}(\Delta/c).
%   \end{equation}
% \end{corollary}

% \begin{proof}
% Apply \cref{thm:rd-affine-domination} with $b=0$.
% \end{proof}

\begin{theorem}[RD ordering under domination]
  \label{thm:rd-domination}
  If $d_2 \preceq_c d_1$ for some $c>0$, then for every $\Delta\ge 0$,
  \begin{equation}
    \label{eq:rd-order}
    R_{d_2}(\Delta) \;\le\; R_{d_1}(\Delta/c).
  \end{equation}
\end{theorem}

\begin{proof}
  Fix $\Delta\ge 0$ and define feasible sets of reconstruction kernels
  \[
    \mathcal F_2(\Delta) := \{r:\; \mathcal D(r;d_2)\le \Delta\},
    \qquad
    \mathcal F_1(\Delta/c) := \{r:\; \mathcal D(r;d_1)\le \Delta/c\}.
  \]
  Take any $r\in\mathcal F_1(\Delta/c)$. Since $d_2(x,\hat x)\le c\,d_1(x,\hat x)$ pointwise,
  \[
    \mathcal D(r;d_2)=\mathbb E[d_2(X,\hat X)]
    \le c\,\mathbb E[d_1(X,\hat X)]
    = c\,\mathcal D(r;d_1)
    \le c\cdot(\Delta/c)=\Delta.
  \]
  Hence $r\in\mathcal F_2(\Delta)$ and $\mathcal F_1(\Delta/c)\subseteq \mathcal F_2(\Delta)$.
  Taking infima of $I(X;\hat X)$ over the two feasible sets yields \eqref{eq:rd-order}.
\end{proof}

\begin{theorem}[RD ordering under affine domination]
  \label{thm:rd-affine-domination}
  If $d_2 \preceq_{c,b} d_1$ for some $c>0$ and $b\ge 0$, then for every $\Delta\ge b$,
  \begin{equation}
    \label{eq:rd-order-affine}
    R_{d_2}(\Delta) \;\le\; R_{d_1}\!\left(\frac{\Delta-b}{c}\right).
  \end{equation}
\end{theorem}

\begin{proof}
  If $d_2(x,\hat x)\le c\,d_1(x,\hat x)+b$ pointwise, then for any reconstruction kernel $r$,
  \[
    \mathcal D(r;d_2)=\mathbb E[d_2(X,\hat X)]
    \le c\,\mathbb E[d_1(X,\hat X)] + b
    = c\,\mathcal D(r;d_1)+b.
  \]
  Thus any $r$ satisfying $\mathcal D(r;d_1)\le (\Delta-b)/c$ also satisfies $\mathcal D(r;d_2)\le \Delta$.
  Equivalently, $\{r:\mathcal D(r;d_1)\le(\Delta-b)/c\}\subseteq \{r:\mathcal D(r;d_2)\le \Delta\}$.
  Taking infima of $I(X;\hat X)$ over these sets yields \eqref{eq:rd-order-affine}.
\end{proof}

In words, \Cref{thm:rd-domination,thm:rd-affine-domination} formalize the intuition that a pointwise weaker distortion constraint enlarges the feasible set of reconstruction kernels, so the optimal mutual information (rate) cannot increase, up to the corresponding rescaling (and shift) of the distortion budget.

\subsection{When can the RD inequality be strict?}
\label{app:claim1:strict}

Section~\ref{app:claim1:core} establishes only a non-strict RD ordering. We now identify conditions under which this ordering becomes strict, which is the regime relevant to our analysis: a weaker distortion should not merely enlarge the feasible set, but yield a strictly smaller optimal rate at the same numerical distortion budget. The bounds in \Cref{thm:rd-domination,thm:rd-affine-domination} can be tight when the additional components penalized by the stronger distortion can remain at their \emph{zero-rate} optimum without violating the distortion constraint, so an RD-optimal reconstruction need not convey extra information. In contrast, a strict gap arises when the stronger distortion requires improving some component beyond its zero-rate baseline, which is impossible without transmitting information about that component. The next lemma formalizes this mechanism in a simple independence model.

We assume $X=(U,V)$ where $U\in\mathbb{R}^{n_U}$ and $V\in\mathbb{R}^{n_V}$ are random vectors with $U\perp V$.
Consider the weak (feature-like) and strong (pixel-like) distortions
\begin{align}
  d_{\mathrm{weak}}((u,v),(\hat u,\hat v)) &:= \|u-\hat u\|_2^2,\\
  d_{\mathrm{strong}}((u,v),(\hat u,\hat v)) &:= \|u-\hat u\|_2^2 + \|v-\hat v\|_2^2.
\end{align}

Intuitively, $d_{\mathrm{weak}}$ depends only on the $U$-coordinates, so we may choose a reconstruction that ignores $V$ entirely (e.g., by setting $\hat V$ to a constant) without affecting feasibility. The next lemma formalizes this reduction by showing that optimizing the RD tradeoff under $d_{\mathrm{weak}}$ is equivalent to the standard RD problem for the marginal source $U$.

\begin{lemma}[Weak RD equals the RD of $U$]
  \label{lem:weak-equals-u}
  For all $\Delta\ge 0$,
  \begin{equation}
    \label{eq:weak-equals-u}
    R_{d_{\mathrm{weak}}}(\Delta) \;=\; R_{U}(\Delta),
  \end{equation}
  where $R_U(\Delta)$ denotes the Shannon RD function of the source $U$ under squared error $\|u-\hat u\|_2^2$.
\end{lemma}

\begin{proof}
  We prove equality by matching an upper bound and a lower bound.

  \emph{Upper bound.}
  Fix any kernel $r_U(\hat u\mid u)$ such that $\mathbb E\|U-\hat U\|_2^2\le \Delta$.
  Define a kernel for $(\hat U,\hat V)$ given $(U,V)$ by sampling $\hat U\sim r_U(\cdot\mid U)$ and setting $\hat V\equiv \hat v_0$
  to a constant independent of $(U,V)$.
  Then the achieved distortion under $d_{\mathrm{weak}}$ is exactly $\mathbb E\|U-\hat U\|_2^2\le \Delta$,
  so the constructed kernel is feasible for $R_{d_{\mathrm{weak}}}(\Delta)$.
  Moreover, since $\hat V$ is constant and carries no information, and since $U\perp V$,
  \[
    I(U,V;\hat U,\hat V)=I(U,V;\hat U)=I(U;\hat U).
  \]
  Taking the infimum over all such $r_U$ yields $R_{d_{\mathrm{weak}}}(\Delta)\le R_U(\Delta)$.

  \emph{Lower bound.}
  Take any feasible kernel $r(\hat u,\hat v\mid u,v)$ for $R_{d_{\mathrm{weak}}}(\Delta)$.
  Feasibility implies $\mathbb E\|U-\hat U\|_2^2\le \Delta$.
  By the definition of $R_U(\Delta)$, this implies $I(U;\hat U)\ge R_U(\Delta)$.
  Finally, adding variables cannot decrease mutual information, hence
  \[
    I(U,V;\hat U,\hat V)\ge I(U;\hat U)\ge R_U(\Delta).
  \]
  Taking the infimum over all feasible $r$ yields $R_{d_{\mathrm{weak}}}(\Delta)\ge R_U(\Delta)$.

  Combining the two bounds proves \eqref{eq:weak-equals-u}.
\end{proof}

Having reduced the weak distortion to an RD problem on $U$ alone, we now isolate the role of $V$ under the strong distortion. The key quantity is the best achievable squared-error distortion on $V$ when the reconstruction carries \emph{no information} about $V$, i.e., the zero-rate baseline. 

\begin{definition}[Zero-rate distortion level under squared error]
  \label{def:zero-rate-distortion}
  Let $V$ be an $\mathbb{R}^{n_V}$-valued random vector. The \emph{zero-rate distortion level} for $V$ under squared error is
  \begin{equation}
    \label{eq:dv0}
    D_V^{(0)} \;:=\; \inf_{\hat v\in\mathbb{R}^{n_V}} \mathbb{E}\big[\|V-\hat v\|_2^2\big].
  \end{equation}
\end{definition}

The zero-rate level $D_V^{(0)}$ marks the smallest expected error achievable by a constant reconstruction of $V$. If the strong distortion budget forces $\hat V$ to beat this baseline, then $\hat V$ cannot be independent of $V$ and must necessarily encode information about it. The next lemma makes this implication precise.

\begin{lemma}[Beating the zero-rate baseline forces information about $V$]
  \label{lem:positive-info}
  If $\mathbb{E}\|V-\hat V\|_2^2 < D_V^{(0)}$, then $I(V;\hat V)>0$.
\end{lemma}

\begin{proof}
  We prove the contrapositive. If $I(V;\hat V)=0$, then $V$ and $\hat V$ are independent, and
  \[
    \mathbb{E}\|V-\hat V\|_2^2
    = \mathbb{E}_{\hat V}\big[\mathbb{E}\|V-\hat v\|_2^2\big]
    \ge \inf_{\hat v}\mathbb{E}\|V-\hat v\|_2^2
    = D_V^{(0)}.
  \]
\end{proof}

To compare against the weak objective, it is convenient to express this information requirement in a form that interacts with the chain rule over $(U,V)$. Under independence of $U$ and $V$, any dependence between $V$ and $\hat V$ also persists after conditioning on $U$, which we record next.

\begin{lemma}[Unconditional information implies conditional information under independence]
  \label{lem:cond-mi}
  Assume $U\perp V$. If $I(V;\hat V)>0$, then $I(V;\hat V\mid U)>0$.
\end{lemma}

\begin{proof}
  We prove the contrapositive. If $I(V;\hat V\mid U)=0$, then $V\perp \hat V\mid U$.
  Using $U\perp V$ and marginalizing over $U$ implies $V\perp \hat V$, hence $I(V;\hat V)=0$.
\end{proof}

We can now combine these ingredients. When $\Delta < D_V^{(0)}$, feasibility under $d_{\mathrm{strong}}$ forces $\hat V$ to carry positive information about $V$ (even given $U$), while $d_{\mathrm{weak}}$ imposes no such requirement. This yields a strict separation in optimal rates at the same distortion budget.

\begin{definition}[Rate-distortion function of the nuisance variable]
  \label{def:rv}
  Let $R_V(\delta)$ denote the Shannon rate-distortion function of the source $V$
  under squared error $\|v-\hat v\|_2^2$.
\end{definition}

\begin{lemma}[Uniform information cost below the zero-rate level]
  \label{lem:uniform-v-info}
  Assume $U\perp V$. For any reconstruction kernel producing $(\hat U,\hat V)$ from $(U,V)$,
  if
  \[
    \mathbb E\|V-\hat V\|_2^2 \le \delta,
  \]
  then
  \[
    I(V;\hat U,\hat V\mid U) \ge R_V(\delta).
  \]
\end{lemma}

\begin{proof}
  For each value of $U=u$, the conditional law of $(\hat U,\hat V)$ given $V$ defines a reconstruction channel for the source $V$,
  since $U\perp V$. Let
  \[
    \delta(u) := \mathbb E[\|V-\hat V\|_2^2\mid U=u].
  \]
  Then, by the definition of $R_V$,
  \[
    I(V;\hat U,\hat V\mid U=u) \ge R_V(\delta(u)).
  \]
  Averaging over $u$ and using convexity of the Shannon RD function gives
  \[
    I(V;\hat U,\hat V\mid U)
    \ge
    \mathbb E_U[R_V(\delta(U))]
    \ge
    R_V(\mathbb E_U[\delta(U)])
    \ge
    R_V(\delta),
  \]
  where the final inequality uses monotonicity of $R_V$.
\end{proof}

\begin{theorem}[Strict RD gap under nuisance variability]
  \label{thm:strict-gap}
  Assume $U\perp V$ and use $d_{\mathrm{weak}},d_{\mathrm{strong}}$ as above.
  % If $\Delta < D_V^{(0)}$, then
  If $\Delta < D_V^{(0)}$ and $R_V(\Delta)>0$, then
  \begin{equation}
    \label{eq:strict-gap}
    R_{d_{\mathrm{strong}}}(\Delta) \;>\; R_{d_{\mathrm{weak}}}(\Delta).
  \end{equation}
\end{theorem}

\begin{proof}
  Fix $\Delta < D_V^{(0)}$ and take any kernel feasible for $R_{d_{\mathrm{strong}}}(\Delta)$:
  \[
    \mathbb E\|U-\hat U\|_2^2 + \mathbb E\|V-\hat V\|_2^2 \le \Delta.
  \]
  In particular,
  \[
    \mathbb E\|U-\hat U\|_2^2 \le \Delta,
    \qquad
    \mathbb E\|V-\hat V\|_2^2 \le \Delta.
  \]
  By the chain rule and monotonicity of mutual information,
  \[
    I(U,V;\hat U,\hat V)
    =
    I(U;\hat U,\hat V)
    +
    I(V;\hat U,\hat V\mid U)
    \ge
    I(U;\hat U)
    +
    I(V;\hat U,\hat V\mid U).
  \]
  Since $\mathbb E\|U-\hat U\|_2^2\le\Delta$,
  \[
    I(U;\hat U)\ge R_U(\Delta).
  \]
  Since $\mathbb E\|V-\hat V\|_2^2\le\Delta$, Lemma~\ref{lem:uniform-v-info} gives
  \[
    I(V;\hat U,\hat V\mid U)\ge R_V(\Delta).
  \]
  Therefore every feasible kernel satisfies
  \[
    I(U,V;\hat U,\hat V)
    \ge
    R_U(\Delta)+R_V(\Delta).
  \]
  Taking the infimum over feasible kernels yields
  \[
    R_{d_{\mathrm{strong}}}(\Delta)
    \ge
    R_U(\Delta)+R_V(\Delta)
    >
    R_U(\Delta)
    =
    R_{d_{\mathrm{weak}}}(\Delta),
  \]
  where the strict inequality uses $R_V(\Delta)>0$ and the final equality follows from
  Lemma~\ref{lem:weak-equals-u}.
\end{proof}
% \begin{proof}
%   Fix $\Delta < D_V^{(0)}$ and take any kernel feasible for $R_{d_{\mathrm{strong}}}(\Delta)$:
%   \[
%     \mathbb{E}\|U-\hat U\|_2^2 + \mathbb{E}\|V-\hat V\|_2^2 \le \Delta.
%   \]
%   Then $\mathbb{E}\|V-\hat V\|_2^2 \le \Delta < D_V^{(0)}$, so \cref{lem:positive-info} gives $I(V;\hat V)>0$,
%   and \cref{lem:cond-mi} gives $I(V;\hat V\mid U)>0$.
%   By the chain rule,
%   \[
%     I(U,V;\hat U,\hat V)
%     = I(U;\hat U,\hat V) + I(V;\hat U,\hat V\mid U)
%     \ge I(U;\hat U) + I(V;\hat V\mid U)
%     > I(U;\hat U).
%   \]
%   Feasibility implies $\mathbb{E}\|U-\hat U\|_2^2 \le \Delta$, hence $I(U;\hat U)\ge R_U(\Delta)$ by definition of $R_U$.
%   Therefore $I(U,V;\hat U,\hat V)>R_U(\Delta)$ for every feasible kernel; taking infima yields
%   \[
%     R_{d_{\mathrm{strong}}}(\Delta)>R_U(\Delta)=R_{d_{\mathrm{weak}}}(\Delta),
%   \]
%   where the last equality follows from \cref{lem:weak-equals-u}.
% \end{proof}

\subsection{Application I: Perceptual (VGG) Feature MSE is Weaker than Pixel MSE}
\label{app:claim1:vgg}

We now instantiate the abstract domination relation Definition~\ref{def:domination} for a standard pair of objectives used in practice: Pixel MSE and Perceptual Feature MSE. The key observation is that if the feature map is Lipschitz on the data domain, then feature-space squared error is pointwise controlled by pixel-space squared error. This immediately translates into an RD ordering via \Cref{thm:rd-domination}.

Let $\mathcal{X}\subset \mathbb{R}^n$ be the data domain and let $\phi:\mathbb{R}^n\to\mathbb{R}^m$ be a deterministic feature map.

\begin{definition}[Pixel distortion]
  \label{def:pix-dist}
  The pixel-space squared-error distortion is
  \begin{equation}
    \label{eq:pix-dist}
    d_{\mathrm{pix}}(x,\hat x) := \|x-\hat x\|_2^2,
    \qquad x,\hat x\in\mathcal{X}.
  \end{equation}
\end{definition}

\begin{definition}[Feature distortion]
  \label{def:feat-dist}
  The feature-space squared-error distortion induced by $\phi$ is
  \begin{equation}
    \label{eq:feat-dist}
    d_{\phi}(x,\hat x) := \|\phi(x)-\phi(\hat x)\|_2^2,
    \qquad x,\hat x\in\mathcal{X}.
  \end{equation}
\end{definition}

To relate these distortions, we require a regularity condition on the feature map that rules out arbitrarily large feature changes induced by small pixel perturbations. We capture this with a Lipschitz bound on $\phi$ over the data domain.

\begin{assumption}[Lipschitz continuity on $\mathcal{X}$]
  \label{ass:lipschitz-vgg}
  There exists $L<\infty$ such that
  \begin{equation}
    \label{eq:lipschitz-vgg}
    \|\phi(x)-\phi(\hat x)\|_2 \le L\|x-\hat x\|_2
    \qquad \forall x,\hat x\in\mathcal{X}.
  \end{equation}
\end{assumption}

\begin{remark}
  Assumption~\ref{ass:lipschitz-vgg} is a standard regularity condition for neural feature maps on a bounded domain:
  for networks with Lipschitz nonlinearities, one can bound a global Lipschitz constant by the product of operator norms
  of the linear layers. In practice, such bounds are often enforced or encouraged by explicit constraints/regularizers
  (e.g., spectral normalization \cite{miyato2018spectralnormalizationgenerativeadversarial} in GAN discriminators).
\end{remark}

Under this assumption, the comparison reduces to a one-line inequality: the feature-space error is at most a constant multiple of the pixel-space error. We state this pointwise domination explicitly, since it is the sufficient condition needed by \Cref{thm:rd-domination}.

\begin{lemma}[Pointwise domination]
  \label{lem:vgg-domination}
  Under Assumption~\ref{ass:lipschitz-vgg},
  \begin{equation}
    \label{eq:vgg-domination}
    d_{\phi}(x,\hat x) \;\le\; L^2\, d_{\mathrm{pix}}(x,\hat x)\qquad \forall x,\hat x\in\mathcal{X}.
  \end{equation}
\end{lemma}

\begin{proof}
  Square \eqref{eq:lipschitz-vgg}.
\end{proof}

Pointwise domination implies set inclusion of feasible reconstruction kernels at matched distortion budgets, and therefore an ordering of the corresponding Shannon RD functions. The following corollary records the resulting rate bound, including the explicit rescaling of the distortion axis.

\begin{corollary}[RD ordering for VGG feature MSE vs.\ pixel MSE]
  \label{cor:vgg-rd}
  Under Assumption~\ref{ass:lipschitz-vgg}, for all $\Delta\ge 0$,
  \begin{equation}
    \label{eq:vgg-rd}
    R_{d_{\phi}}(\Delta) \;\le\; R_{d_{\mathrm{pix}}}(\Delta/L^2).
  \end{equation}
  Equivalently, with $\tilde d_{\phi}:=d_\phi/L^2$,
  \begin{equation}
    \label{eq:vgg-rd-normalized}
    R_{\tilde d_{\phi}}(\Delta) \;\le\; R_{d_{\mathrm{pix}}}(\Delta).
  \end{equation}
\end{corollary}

\begin{proof}
  Lemma~\ref{lem:vgg-domination} is exactly $d_\phi \preceq_{L^2} d_{\mathrm{pix}}$ in the sense of Definition~\ref{def:domination}.
  Apply \Cref{thm:rd-domination} to obtain \eqref{eq:vgg-rd}.
\end{proof}

\subsection{Application II: Discriminator-Based Losses}
\label{app:claim1:disc}

We now extend the same comparison principle to discriminator-based objectives. The main step is to reinterpret commonly used generator-side losses as pairwise functions of $(x,\hat x)$ and then upper bound them (up to constants) by pixel MSE using standard Lipschitz-type regularity. Once such a pointwise bound is in place, the RD ordering follows directly from \Cref{thm:rd-domination,thm:rd-affine-domination}.

\subsubsection{Feature matching: intermediate discriminator features}
\label{app:claim1:disc-fm}
Let's start by defining feature matching loss.

\begin{definition}[Feature matching distortion]
  \label{def:fm-dist}
  Let $F_{\psi,\ell}:\mathcal{X}\to\mathbb{R}^{m_\ell}$ denote the activation vector at discriminator layer $\ell$.
  Given a set of layers $\mathcal{L}$ and weights $\{w_\ell\}_{\ell\in\mathcal{L}}$ with $w_\ell\ge 0$, the
  \emph{feature matching distortion} is defined by
  \begin{equation}
    \label{eq:fm-dist}
    d_{\mathrm{FM}}(x,\hat x)
    := \sum_{\ell\in\mathcal{L}} w_\ell \,\big\|F_{\psi,\ell}(x)-F_{\psi,\ell}(\hat x)\big\|_2^2.
  \end{equation}
\end{definition}

Again, to relate $d_{\mathrm{FM}}$ to pixel MSE, it suffices to control how much intermediate discriminator features can change under a perturbation in input space. We capture this via layer-wise Lipschitz bounds on $F_{\psi,\ell}$ over the data domain.

\begin{assumption}[Layer-wise Lipschitz discriminator features on $\mathcal{X}$]
  \label{ass:lipschitz-fm}
  For each $\ell\in\mathcal{L}$ there exists $L_\ell<\infty$ such that
  \begin{equation}
    \label{eq:lipschitz-fm}
    \big\|F_{\psi,\ell}(x)-F_{\psi,\ell}(\hat x)\big\|_2 \le L_\ell \|x-\hat x\|_2
    \qquad \forall x,\hat x\in\mathcal{X}.
  \end{equation}
\end{assumption}

\begin{remark}
  As in Assumption~\ref{ass:lipschitz-vgg}, Assumption~\ref{ass:lipschitz-fm} holds whenever each layer up to $\ell$
  has bounded operator norm and uses Lipschitz nonlinearities; it is common in practice to explicitly regularize
  discriminator Lipschitz constants (e.g., via spectral normalization).
\end{remark}

Under these bounds, each feature discrepancy is dominated by a constant multiple of the pixel discrepancy. Summing across layers yields a single domination constant for the entire feature matching objective.

\begin{lemma}[Feature matching is weaker than pixel MSE]
  \label{lem:fm-domination}
  Under Assumption~\ref{ass:lipschitz-fm},
  \begin{equation}
    \label{eq:fm-domination}
    d_{\mathrm{FM}}(x,\hat x) \;\le\; \Big(\sum_{\ell\in\mathcal{L}} w_\ell L_\ell^2\Big)\,\|x-\hat x\|_2^2
    \qquad \forall x,\hat x\in\mathcal{X}.
  \end{equation}
  Equivalently, $d_{\mathrm{FM}} \preceq_{c_{\mathrm{FM}}} d_{\mathrm{pix}}$ with
  $c_{\mathrm{FM}} := \sum_{\ell\in\mathcal{L}} w_\ell L_\ell^2$.
\end{lemma}

\begin{proof}
  For each $\ell$, \eqref{eq:lipschitz-fm} implies
  $\|F_{\psi,\ell}(x)-F_{\psi,\ell}(\hat x)\|_2^2 \le L_\ell^2\|x-\hat x\|_2^2$.
  Multiply by $w_\ell$ and sum over $\ell\in\mathcal{L}$.
\end{proof}

This pointwise domination implies that any reconstruction rule meeting a pixel-MSE budget also meets an appropriately rescaled feature-matching budget. The corresponding RD ordering is an immediate application of \Cref{thm:rd-domination}.

\begin{corollary}[RD ordering for feature matching vs.\ pixel MSE]
  \label{cor:fm-rd}
  Under Assumption~\ref{ass:lipschitz-fm}, for all $\Delta\ge 0$,
  \begin{equation}
    \label{eq:fm-rd}
    R_{d_{\mathrm{FM}}}(\Delta) \;\le\; R_{d_{\mathrm{pix}}}\!\left(\frac{\Delta}{c_{\mathrm{FM}}}\right),
    \qquad c_{\mathrm{FM}}=\sum_{\ell\in\mathcal{L}} w_\ell L_\ell^2.
  \end{equation}
\end{corollary}

\begin{proof}
  Apply \Cref{thm:rd-domination} to Lemma~\ref{lem:fm-domination}.
\end{proof}

\subsubsection{Hinge generator loss: PatchGAN logits}
\label{app:claim1:disc-hinge}
We next consider the hinge generator objective. Unlike feature matching, the hinge loss is naturally defined on single samples; we therefore first convert it into a pairwise distortion by taking score differences, and then enforce nonnegativity by a constant shift. This produces a valid RD distortion while preserving the optimizer of the expected objective up to an additive constant.

%%%%%% LOWER REDO %%%%%%%%
\begin{definition}[PatchGAN scalar score and hinge generator loss]
  \label{def:patch-hinge}
  Let $D_\psi:\mathcal{X}\to\mathbb{R}^{H'\times W'}$ be a PatchGAN discriminator returning a grid of logits.
  Define the associated scalar score by
  \begin{equation}
    \label{eq:patch-score}
    s_\psi(x) := \frac{1}{H'W'}\sum_{i=1}^{H'}\sum_{j=1}^{W'} (D_\psi(x))_{ij}.
  \end{equation}
  The (generator) hinge loss is then
  \begin{equation}
    \label{eq:hinge-gen}
    \ell_{\mathrm{G}}(\hat x) := -\,s_\psi(\hat x).
  \end{equation}
\end{definition}

\begin{definition}[Score-difference Quantity]
    For paired samples $(x,\hat x)$ define the score-difference quantity
\begin{equation}
  \label{eq:hinge-raw}
  d_{\mathrm{hinge}}^{\mathrm{raw}}(x,\hat x)
  := \ell_{\mathrm{G}}(\hat x)-\ell_{\mathrm{G}}(x)
  = s_\psi(x)-s_\psi(\hat x).
\end{equation}
\end{definition}

Taking expectations yields
\begin{equation}
\mathbb{E}\!\left[d_{\mathrm{hinge}}^{\mathrm{raw}}(X,\hat X)\right]
= \mathbb{E}\!\left[\ell_{\mathrm{G}}(\hat X)\right]-\mathbb{E}\!\left[\ell_{\mathrm{G}}(X)\right].
\end{equation}

Since $X\sim p_X$ is fixed, the term $\mathbb{E}[\ell_{\mathrm{G}}(X)]$ is constant with respect to the reconstruction kernel. Consequently, minimizing $\mathbb{E}[\ell_{\mathrm{G}}(\hat X)]$ is equivalent to minimizing $\mathbb{E}[d_{\mathrm{hinge}}^{\mathrm{raw}}(X,\hat X)]$.

This difference form isolates exactly the component of the objective that depends on the reconstruction kernel. We therefore take the pairwise score difference as the fundamental quantity to be bounded.

% Then $\mathbb{E}[d_{\mathrm{hinge}}^{\mathrm{raw}}(X,\hat X)]
% = \mathbb{E}[\ell_{\mathrm{G}}(\hat X)]-\mathbb{E}[\ell_{\mathrm{G}}(X)]$,
% so minimizing $\mathbb{E}[\ell_{\mathrm{G}}(\hat X)]$ is equivalent to minimizing
% $\mathbb{E}[d_{\mathrm{hinge}}^{\mathrm{raw}}(X,\hat X)]$ since the second term is constant
% with respect to the reconstruction kernel.

% The difference form isolates the part of the objective that depends on the reconstruction kernel, since the term involving $x$ is constant under $p_X$. We will therefore work with the pairwise score difference as the basic quantity to be bounded.

%%%%%%% UPPER %%%%%%%%%

\paragraph{Nonnegativity and why shifting is the right abstraction.}
In RD theory, a distortion is required to be nonnegative pointwise.
The raw difference $d_{\mathrm{hinge}}^{\mathrm{raw}}(x,\hat x)$ is not guaranteed to be nonnegative for all $(x,\hat x)$,
especially early in training.
This is not an obstacle for our purposes: adding a constant shift produces a valid nonnegative distortion,
and its expectation differs from the original generator objective only by an additive constant.

\begin{assumption}[Bounded data domain]
  \label{ass:bounded-domain-hinge}
  The data domain is bounded, e.g.\ $\mathcal{X}\subset[-1,1]^n$.
\end{assumption}

Define the score range on $\mathcal{X}$:
\begin{equation}
  \label{eq:score-range}
  \Delta_\psi := \sup_{x\in\mathcal{X}} s_\psi(x) - \inf_{x\in\mathcal{X}} s_\psi(x) < \infty.
\end{equation}
Then the shifted distortion
\begin{equation}
  \label{eq:hinge-shift}
  \bar d_{\mathrm{hinge}}(x,\hat x) := d_{\mathrm{hinge}}^{\mathrm{raw}}(x,\hat x) + \Delta_\psi
\end{equation}
satisfies $\bar d_{\mathrm{hinge}}(x,\hat x)\ge 0$ for all $(x,\hat x)$.

To compare $\bar d_{\mathrm{hinge}}$ to pixel MSE, we require a stability condition on the discriminator score: small changes in the input should not induce arbitrarily large changes in the score. We encode this with a Lipschitz bound on $s_\psi$ over $\mathcal{X}$.

\begin{assumption}[Lipschitz discriminator score on $\mathcal{X}$]
  \label{ass:lipschitz-score-hinge}
  There exists $L_\psi<\infty$ such that
  \begin{equation}
    \label{eq:lipschitz-score}
    |s_\psi(x)-s_\psi(\hat x)| \le L_\psi \,\|x-\hat x\|_2
    \qquad \forall x,\hat x\in\mathcal{X}.
  \end{equation}
\end{assumption}

With bounded score range and Lipschitz continuity, $\bar d_{\mathrm{hinge}}$ admits an affine upper bound in terms of $\|x-\hat x\|_2^2$ by a standard Young's inequality argument. This places the hinge distortion within the affine-domination framework.

\begin{lemma}[Affine domination of shifted hinge distortion by pixel MSE]
  \label{lem:hinge-affine-domination}
  Under Assumptions~\ref{ass:bounded-domain-hinge} and~\ref{ass:lipschitz-score-hinge}, there exist $c>0$ and $b\ge 0$ such that
  \begin{equation}
    \label{eq:hinge-affine-domination}
    \bar d_{\mathrm{hinge}}(x,\hat x) \le c\,\|x-\hat x\|_2^2 + b
    \qquad \forall x,\hat x\in\mathcal{X}.
  \end{equation}
  In particular, one valid choice is $c=L_\psi^2/2$ and $b=\Delta_\psi + 1/2$.
\end{lemma}

\begin{proof}
  By \eqref{eq:hinge-shift} and the triangle inequality,
  \[
    \bar d_{\mathrm{hinge}}(x,\hat x)
    = s_\psi(x)-s_\psi(\hat x) + \Delta_\psi
    \le |s_\psi(x)-s_\psi(\hat x)| + \Delta_\psi.
  \]
  By \eqref{eq:lipschitz-score},
  \[
    \bar d_{\mathrm{hinge}}(x,\hat x)\le \Delta_\psi + L_\psi \|x-\hat x\|_2.
  \]
  Apply Young's inequality $a t \le \frac{a^2}{2}t^2 + \frac{1}{2}$ with $a=L_\psi$ and $t=\|x-\hat x\|_2$ to obtain
  \[
    L_\psi \|x-\hat x\|_2 \le \frac{L_\psi^2}{2}\|x-\hat x\|_2^2 + \frac{1}{2}.
  \]
  Combining the last two displays yields \eqref{eq:hinge-affine-domination} with $c=L_\psi^2/2$ and $b=\Delta_\psi+1/2$.
\end{proof}

Affine domination again implies inclusion of feasible sets after an appropriate shift and rescaling of the distortion budget. The resulting RD bound follows by invoking \Cref{thm:rd-affine-domination}.

\begin{corollary}[RD ordering for shifted hinge distortion vs.\ pixel MSE]
  \label{cor:hinge-rd}
  Under the assumptions of Lemma~\ref{lem:hinge-affine-domination}, for all $\Delta\ge b$,
  \begin{equation}
    \label{eq:hinge-rd}
    R_{\bar d_{\mathrm{hinge}}}(\Delta) \;\le\; R_{d_{\mathrm{pix}}}\!\left(\frac{\Delta-b}{c}\right),
    \qquad c=\frac{L_\psi^2}{2},\;\; b=\Delta_\psi+\frac{1}{2}.
  \end{equation}
\end{corollary}

\begin{proof}
  Lemma~\ref{lem:hinge-affine-domination} states $\bar d_{\mathrm{hinge}} \preceq_{c,b} d_{\mathrm{pix}}$ in the sense of Definition~\ref{def:affine-domination}.
  Apply \Cref{thm:rd-affine-domination}.
\end{proof}

\subsection{Generalization to common neural losses}
\label{app:generalization}

The preceding derivations are not specific to VGG or PatchGAN. They only require the neural distortion to be weaker than pixel SSE in the sense of Definition~\ref{def:affine-domination}. Concretely, let \(d_N\) be any neural reconstruction term. If there exist constants \(c_N>0\) and \(b_N\ge 0\) such that
\[
d_N(x,\hat{x}) \le c_N d_{\mathrm{pix}}(x,\hat{x}) + b_N
\qquad \forall x,\hat{x}\in\mathcal X,
\]
then \(d_N \preceq_{c_N,b_N} d_{\mathrm{pix}}\), and Theorem~\ref{thm:rd-affine-domination} gives
\[
R_{d_N}(\Delta)
\le
R_{d_{\mathrm{pix}}}\!\left(\frac{\Delta-b_N}{c_N}\right),
\qquad \Delta \ge b_N .
\]

This condition is architecture-agnostic. It holds for the usual feature-space losses whenever the admissible image domain \(\mathcal X\) is bounded, the feature maps used by the loss are Lipschitz on \(\mathcal X\), and the loss combines them with nonnegative weights. In particular, any loss of the form
\[
d_N(x,\hat{x})
=
\sum_j w_j \| f_j(x)-f_j(\hat{x}) \|_2^2,
\qquad w_j\ge 0,
\]
satisfies the condition with \(b_N=0\) whenever each \(f_j\) is \(L_j\)-Lipschitz, since then
\[
d_N(x,\hat{x})
\le
\left(\sum_j w_j L_j^2\right)d_{\mathrm{pix}}(x,\hat{x}).
\]
This covers standard perceptual-feature losses, discriminator feature matching, and representation losses based on fixed neural encoders.

Score-based adversarial losses fit the same framework after rewriting the single-sample generator objective as a pairwise score difference and adding a constant shift, as in Section~\ref{app:claim1:disc-hinge}. The required assumptions are again only boundedness of \(\mathcal X\), finite score range, and Lipschitz continuity of the score function on \(\mathcal X\).

Finally, the class is closed under nonnegative mixtures: if \(d_j\preceq_{c_j,b_j}d_{\mathrm{pix}}\) and \(a_j\ge 0\), then
\[
\sum_j a_j d_j
\preceq_{\sum_j a_j c_j,\ \sum_j a_j b_j}
d_{\mathrm{pix}} .
\]
Thus the RD ordering applies to the mixed reconstruction objectives used in practice, independently of the particular neural architecture, as long as these regularity and nonnegativity assumptions hold. As throughout this appendix, this is a statement about the ideal Shannon RD function; the achieved KL of a finite VAE also depends on amortization, optimization, and loss scaling.

\newpage

\FloatBarrier
% ============================================================
% CLAIM 2
% ============================================================

\section{Distortion shapes latent uncertainty}
\label{app:claim2-distortion-shapes-uncertainty}

This section supports the claim in Section~\ref{sec:claim2}: a fixed achieved rate does \emph{not} uniquely determine
how uncertainty is distributed across latent dimensions, and the choice of distortion can directly shape the posterior's
\emph{uncertainty profile}. The formal results below are proved under simplified assumptions; in particular,
\ref{app:toymodel} analyzes an exactly solvable linear-Gaussian toy model.

\subsection{Rate does not determine anisotropy of the posterior}
\label{app:claim2:rate-not-determine}

We use the notation of Section~\ref{sec:background}. Our goal is to show that the \emph{rate} (in particular, the variance contribution to the KL) does not uniquely determine the shape of an individual diagonal posterior. Concretely, the variance term is a single scalar summary of a $D$-dimensional variance vector, leaving room for substantial redistribution across coordinates. We begin by stating a standard equality in mathematical statistics without proving it again:

\begin{lemma}[KL decomposition for diagonal Gaussian vs.\ isotropic prior]
\label{lem:claim2:kl-decomp}
For $q_\phi(z\mid x)=\mathcal N(\mu_\phi(x),\mathrm{diag}(\sigma_\phi^2(x)))$ and $p(z)=\mathcal N(0,I_D)$,
\begin{equation}
  \label{eq:claim2:kl-decomp}
\mathrm{KL}\big(q_\phi(z\mid x)\,\|\,p(z)\big)
=\underbrace{\frac12\|\mu_\phi(x)\|_2^2}_{:=\mathrm{KL}_\mu(x)}
+
\underbrace{\frac12\sum_{i=1}^D\big(\sigma_{\phi,i}^2(x)-1-\log \sigma_{\phi,i}^2(x)\big)}_{:=\mathrm{KL}_{\sigma^2}(x)}.
\end{equation}
\end{lemma}

\begin{remark}
Lemma~\ref{lem:claim2:kl-decomp} follows from the closed-form KL between Gaussians and is included for notation.
\end{remark}

We isolate the variance-dependent part of the KL by introducing a scalar penalty applied coordinate-wise. This lets us view $\mathrm{KL}_{\sigma^2}(x)$ as an additive cost over the per-dimension variances.

\begin{definition}[Scalar variance penalty]
\label{def:claim2:g}
Define $g:(0,\infty)\to\mathbb R$ by $$g(s):=\frac12(s-1-\log s) $$ then
$\mathrm{KL}_{\sigma^2}(x)=\sum_{i=1}^D g(\sigma_{\phi,i}^2(x))$.
\end{definition}

Averaging this variance penalty over the data distribution yields the variance component of the overall rate. We will use it as the scalar quantity that is often implicitly treated as controlling posterior “spread.’’

\begin{definition}[Variance part of the rate]
Define the variance-part rate:
\begin{equation}
  \label{eq:claim2:variance-rate}
  R_{\sigma^2} := \mathbb E_{x\sim p_X}\big[\mathrm{KL}_{\sigma^2}(x)\big].
\end{equation}    
\end{definition}

To capture how the variance is distributed across latent coordinates, we introduce a simple anisotropy statistic for a single posterior covariance. Since the encoder is diagonal, anisotropy reduces to variation in the log-variances across coordinates.

\begin{definition}[Per-sample posterior-variance anisotropy]
  \label{def:claim2:post-aniso}
  For $s\in(0,\infty)^D$, define
  \begin{equation}
    \label{eq:claim2:post-aniso}
    A_{\mathrm{post}}(s)
    :=
    \mathrm{Var}_{i\in\{1,\dots,D\}}\big(\log s_i\big).
  \end{equation}
  For a data point $x$, we write $A_{\mathrm{post}}(x):=A_{\mathrm{post}}(\sigma_\phi^2(x))$.
  For $D=2$, $A_{\mathrm{post}}(s)=\tfrac14(\log s_1-\log s_2)^2$.
\end{definition}

\begin{remark}[Single-posterior vs aggregated-posterior anisotropy]
  \label{rem:claim2:single-vs-agg}
  The statistic $A_{\mathrm{post}}(x)$ measures anisotropy of a \emph{single} diagonal posterior covariance.
  This is distinct from the aggregated-posterior anisotropy $A_{\mathrm{shape}}(\Sigma)$,
  which measures eigenvalue non-uniformity of the dataset-level covariance $\Sigma=\mathrm{Cov}_{q_\phi(z)}[Z]$.
\end{remark}

With these definitions in place, the key point is that $\mathrm{KL}_{\sigma^2}(x)$ constrains only the \emph{sum} of coordinate-wise penalties. Consequently, many distinct variance vectors can achieve the same variance-KL while exhibiting different dispersion patterns across dimensions. The quantity $\mathrm{KL}_{\sigma^2}(x)=\sum_i g(\sigma_{\phi,i}^2(x))$ is a single scalar constraint on a $D$-vector of variances.
For $D\ge 2$, one can redistribute this ``variance-KL budget'' across coordinates while changing the dispersion of
$\{\log \sigma_{\phi,i}^2(x)\}_i$. The next lemma formalizes this degrees-of-freedom argument: even at fixed variance-KL budget, the per-sample posterior can be made more or less anisotropic by reallocating variance across coordinates.

\begin{remark}[Fixed variance-KL does not fix posterior anisotropy]
  \label{lem:claim2:levelset}
  Fix $D\ge 2$ and any constant $C>0$. There exist $s,s'\in(0,\infty)^D$ such that
  \[
    \sum_{i=1}^D g(s_i)=\sum_{i=1}^D g(s'_i)=C,
\qquad\text{but}\qquad
    A_{\mathrm{post}}(s)\neq A_{\mathrm{post}}(s').
\]
\end{remark}

\subsection{Toy Model of Anisotropy}
\label{app:toymodel}

In this section we instantiate the intuition of \Cref{sec:claim2} in a fully solvable linear--Gaussian model.
We show that the reconstruction distortion can directly fix the posterior variance profile, and hence anisotropy, even when the variance part of the KL is matched.

\paragraph{Setup.}
Fix a data point $x\in\mathbb{R}^n$ and a posterior mean $\mu\in\mathbb{R}^D$.
We consider diagonal-Gaussian posteriors
\[
  q(z\mid x) = \mathcal N\!\big(\mu,\operatorname{diag}(s)\big),
  \qquad
  s=(s_1,\dots,s_D)\in(0,\infty)^D.
\]
The decoder is linear,
\[
  \hat x(z) = Wz,\qquad W\in\mathbb{R}^{n\times D},
\]
and reconstructions are evaluated through a squared metric
\[
  d_M(x,z) := \big\|M(x-Wz)\big\|_2^2,
  \qquad M\in\mathbb{R}^{m\times n}.
\]
For fixed $x,\mu,W,M$ and KL weight $\beta>0$ we optimize the ``variance-only'' objective
\begin{equation}
  \label{eq:claim2:variance-only-obj-final}
  \mathcal L_{M,\beta}(s)
  := \mathbb E_{Z\sim\mathcal N(\mu,\operatorname{diag}(s))}[d_M(x,Z)]
     + \beta\sum_{i=1}^D g(s_i),
\end{equation}
where $g(s)=\tfrac12(s-1-\log s)$ is the variance contribution to the KL (see \eqref{eq:claim2:kl-decomp}).

\begin{proposition}[Expected distortion and closed-form optimizer]
  \label{prop:claim2:opt}
  Let $G:=M^\top M\succeq 0$ and $B:=W^\top G W\succeq 0$, and denote $c_i := B_{ii}\ge 0$.
  Then:
  \begin{enumerate}
    \item For $Z\sim\mathcal N(\mu,\operatorname{diag}(s))$,
    \begin{equation}
      \label{eq:claim2:expect-final}
      \mathbb E[d_M(x,Z)]
      = (x-W\mu)^\top G (x-W\mu)
        + \sum_{i=1}^D c_i s_i
      = \mathrm{const}(x,\mu) + \sum_{i=1}^D c_i s_i .
    \end{equation}
    \item The objective can be written as
    \begin{equation}
      \label{eq:claim2:separable-final}
      \mathcal L_{M,\beta}(s)
      = \mathrm{const}(x,\mu) + \sum_{i=1}^D \ell_i(s_i),
      \qquad
      \ell_i(s) := c_i s + \beta g(s).
    \end{equation}
    Each $\ell_i$ is strictly convex on $(0,\infty)$, hence $\mathcal L_{M,\beta}$ has a unique minimizer $s^*(M,\beta)\in(0,\infty)^D$.
    \item The minimizer is given in closed form by
    \begin{equation}
      \label{eq:claim2:s-star-final}
      s_i^*(M,\beta)
      = \frac{1}{1+2c_i/\beta}
      = \frac{\beta}{\beta+2c_i},
      \qquad i=1,\dots,D.
    \end{equation}
  \end{enumerate}
\end{proposition}

\begin{proof}[Proof]
  Write $Z=\mu+\varepsilon$ with $\varepsilon\sim\mathcal N\!\big(0,\operatorname{diag}(s)\big)$ and expand
  \[
    d_M(x,Z) = (x-W\mu-W\varepsilon)^\top G (x-W\mu-W\varepsilon).
  \]
  The cross term vanishes in expectation, and
  $\mathbb E[\varepsilon^\top B\varepsilon]
   = \operatorname{tr}\big(B\,\operatorname{diag}(s)\big)
   = \sum_i c_i s_i$, which yields \eqref{eq:claim2:expect-final} and the separable form \eqref{eq:claim2:separable-final}.

  Since $g''(s)=1/(2s^2)>0$ for $s>0$, each $\ell_i$ is strictly convex, diverges at $0^+$ and $+\infty$, and therefore admits a unique minimizer on $(0,\infty)$; this gives existence and uniqueness of $s^*$.
  Differentiating
  $\ell_i'(s)=c_i+\frac{\beta}{2}(1-1/s)$,
  setting $\ell_i'(s)=0$, and solving for $s$ yields \eqref{eq:claim2:s-star-final}.
  Strict convexity ensures that this critical point is the global minimizer.
\end{proof}

We use the following anisotropy functional on variance profiles:
\begin{equation}
  \label{eq:claim2:Apost-def}
  A_{\mathrm{post}}(s)
  := \operatorname{Var}_{i=1,\dots,D}\big(\log s_i\big),
\end{equation}
which vanishes iff all $s_i$ are equal and grows with the spread of the log-variances.

\begin{corollary}[Distortion controls posterior-variance anisotropy]
  \label{cor:claim2:anis}
  Let $s^*(M,\beta)$ be as in Proposition~\ref{prop:claim2:opt}.
  \begin{enumerate}
    \item If all $c_i$ are equal, then all $s_i^*(M,\beta)$ coincide and $A_{\mathrm{post}}\big(s^*(M,\beta)\big)=0$ (isotropic posterior).
    \item If the coefficients $c_i$ are not all equal, then the entries of $s^*(M,\beta)$ are not all equal and $A_{\mathrm{post}}\big(s^*(M,\beta)\big)>0$ (anisotropic posterior).
  \end{enumerate}
\end{corollary}

We now show that, for a fixed distortion $M$, every positive target value of the variance part of the KL can be obtained by a suitable choice of $\beta$.

\begin{theorem}[Surjectivity of the variance-KL map]
  \label{thm:claim2:surjectivity}
  Fix $W,M$ and assume at least one $c_i>0$.
  Define, for $\beta>0$,
  \begin{equation}
    \label{eq:claim2:C-M-def}
    C_M(\beta)
    := \sum_{i=1}^D g\big(s_i^*(M,\beta)\big),
    \qquad
    s_i^*(M,\beta) = \frac{\beta}{\beta+2c_i}.
  \end{equation}
  Then $C_M:(0,\infty)\to(0,\infty)$ is continuous and strictly decreasing, with
  \[
    \lim_{\beta\to\infty} C_M(\beta) = 0,
    \qquad
    \lim_{\beta\to 0^+} C_M(\beta) = +\infty.
  \]
  Consequently, for every target variance-KL level $C_0>0$ there exists a unique $\beta>0$ such that $C_M(\beta)=C_0$.
\end{theorem}

\begin{proof}[Proof]
  From \eqref{eq:claim2:s-star-final},
  $s_i^*(M,\beta)$ is continuous in $\beta$ and satisfies
  $0<s_i^*(M,\beta)\le 1$ for all $\beta>0$, with
  $s_i^*(M,\beta)\uparrow 1$ as $\beta\to\infty$ and
  $s_i^*(M,\beta)\downarrow 0$ as $\beta\to 0^+$ whenever $c_i>0$.
  On $(0,1]$ the function $g$ is continuous and strictly decreasing, hence
  $g\big(s_i^*(M,\beta)\big)$ is continuous and strictly decreasing in $\beta$ for any $i$ with $c_i>0$.
  Summing over $i$ yields continuity and strict monotonicity of $C_M(\beta)$.

  The limits follow from $g(1)=0$ and $g(s)\to+\infty$ as $s\to 0^+$:
  as $\beta\to\infty$, all $s_i^*(M,\beta)\to 1$ and $C_M(\beta)\to 0$; as $\beta\to 0^+$, any coordinate with $c_i>0$ has $s_i^*(M,\beta)\to 0^+$, forcing $C_M(\beta)\to+\infty$.
  Continuity and strict monotonicity imply that $C_M$ is a bijection from $(0,\infty)$ onto $(0,\infty)$.
\end{proof}

Finally, we combine Corollary~\ref{cor:claim2:anis} and Theorem~\ref{thm:claim2:surjectivity} to construct two models that achieve the same variance-KL but have different anisotropy.

\begin{theorem}[Same variance-KL, different anisotropy in the toy model]
  \label{thm:claim2:toy-main}
  Fix $D\ge 2$ and let $W=I_D$.
  There exist two distortion matrices $M_{\mathrm{iso}}$ and $M_{\mathrm{aniso}}$ and two KL weights $\beta_{\mathrm{iso}},\beta_{\mathrm{aniso}}>0$ such that
  \begin{align}
    \sum_{i=1}^D g\big(s_i^*(M_{\mathrm{iso}},\beta_{\mathrm{iso}})\big)
    &=
    \sum_{i=1}^D g\big(s_i^*(M_{\mathrm{aniso}},\beta_{\mathrm{aniso}})\big),
    \label{eq:claim2:equal-rate}
    \\
    A_{\mathrm{post}}\big(s^*(M_{\mathrm{iso}},\beta_{\mathrm{iso}})\big)
    &\neq
    A_{\mathrm{post}}\big(s^*(M_{\mathrm{aniso}},\beta_{\mathrm{aniso}})\big).
    \label{eq:claim2:different-anisotropy}
  \end{align}
\end{theorem}

\begin{proof}
  Take $M_{\mathrm{iso}} = I_D$.
  Then $G=I_D$ and $B=W^\top G W = I_D$, so $c_i^{\mathrm{iso}} = 1$ for all $i$.
  By Corollary~\ref{cor:claim2:anis}, the corresponding optimizer $s^*(M_{\mathrm{iso}},\beta)$ is isotropic and satisfies
  $A_{\mathrm{post}}\big(s^*(M_{\mathrm{iso}},\beta)\big)=0$ for every $\beta>0$.

  Next choose any $\alpha>0$ with $\alpha\neq 1$ and set
  \[
    M_{\mathrm{aniso}} = \operatorname{diag}(1,\alpha,1,\dots,1).
  \]
  With $W=I_D$, the associated coefficients are
  $c_1^{\mathrm{aniso}} = 1$, $c_2^{\mathrm{aniso}} = \alpha^2$, and $c_i^{\mathrm{aniso}} = 1$ for $i\ge 3$, so they are not all equal.
  By Corollary~\ref{cor:claim2:anis}, $A_{\mathrm{post}}\big(s^*(M_{\mathrm{aniso}},\beta)\big)>0$ for every $\beta>0$.

  Let $C_0>0$ be arbitrary.
  Applying \Cref{thm:claim2:surjectivity} to each distortion separately yields unique
  $\beta_{\mathrm{iso}},\beta_{\mathrm{aniso}}>0$ such that
  \[
    C_{M_{\mathrm{iso}}}(\beta_{\mathrm{iso}})
    =
    C_{M_{\mathrm{aniso}}}(\beta_{\mathrm{aniso}})
    = C_0,
  \]
  which is exactly \eqref{eq:claim2:equal-rate}.
  At these weights, the variance-KL levels match by construction, while
  $A_{\mathrm{post}}$ is zero for $M_{\mathrm{iso}}$ and strictly positive for $M_{\mathrm{aniso}}$, giving \eqref{eq:claim2:different-anisotropy}.
\end{proof}

This toy construction provides a clean analytic example of the phenomenon illustrated empirically in \Cref{sec:claim2}: even when the scalar rate (variance-KL) is held fixed, the distortion function can, through the coefficients $c_i=(W^\top M^\top M W)_{ii}$, reshape the uncertainty profile and anisotropy of the learned posterior.

\begin{figure}[t]
  \centering
  \includegraphics[width=0.8\linewidth]{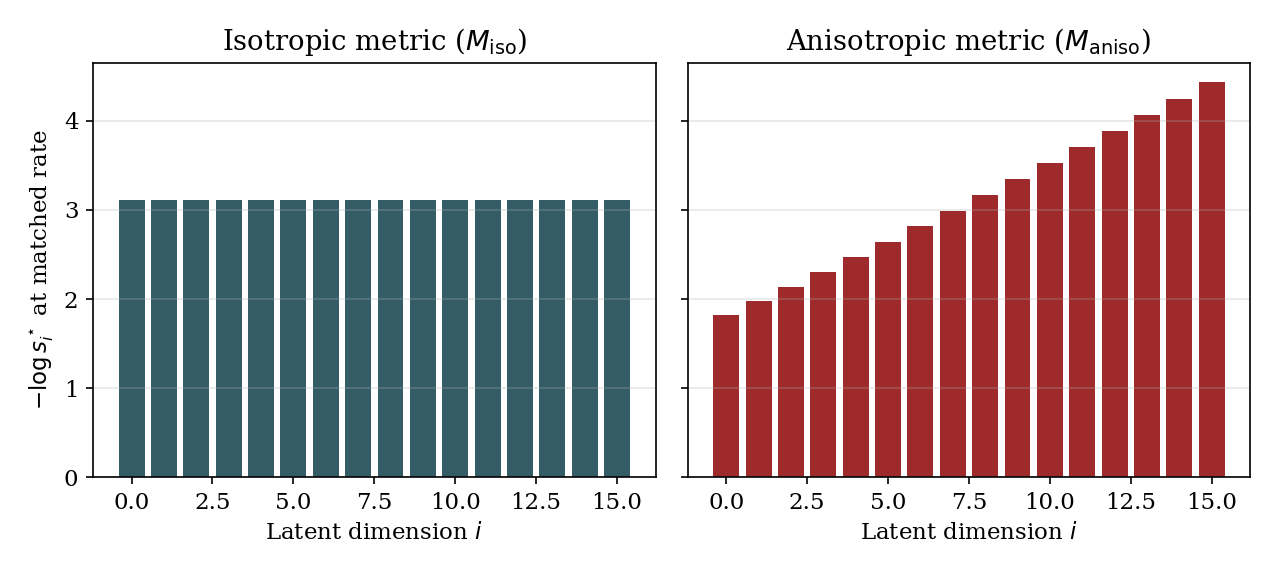}
  \caption{\textbf{Toy linear-Gaussian model: per-dimension certainty at matched rate.}
We instantiate the construction of \Cref{app:toymodel} with $D=16$, $W=I_D$ and compare the optimal posterior variances $s_i^\star(M,\beta)$ obtained for the isotropic metric $M_{\mathrm{iso}}$ and the anisotropic metric $M_{\mathrm{aniso}}$ at a common variance-KL operating point, i.e., for $\beta_{\mathrm{iso}}$ and $\beta_{\mathrm{aniso}}$ such that $\sum_i g(s_i^\star)$ matches the same target level.
For the isotropic metric (left) all bars coincide, confirming that the optimal posterior is isotropic whenever $c_i$ is constant across dimensions (Corollary~\ref{cor:claim2:anis}).
For the anisotropic metric (right), the precision increases smoothly with the latent index, illustrating that changing only the distortion (and not the rate) is sufficient to reshape the uncertainty profile and induce a highly anisotropic posterior, as formalized in \Cref{thm:claim2:toy-main}.}
  \label{fig:barplot}
\end{figure}

\paragraph{Numerical instantiation and summary table.}
The construction of \Cref{thm:claim2:toy-main} is implemented numerically in
\Cref{tab:claim2:toy-results}.
We fix $D=16$ and $W = I_D$, and consider two distortions:
an \emph{isotropic} metric $M_{\mathrm{iso}} = I_D$, which yields
$c_i^{\mathrm{iso}} = 1$ for all $i$, and an \emph{anisotropic} metric
$M_{\mathrm{aniso}} = \mathrm{diag}(w_1,\dots,w_D)$ with
\[
  w_i = \alpha^{(i-1)/(D-1)},\qquad \alpha = 4.0,
\]
which induces a smoothly varying coefficient profile
$c_i^{\mathrm{aniso}} = w_i^2$.
For each metric we evaluate the closed-form solution
$s_i^*(M,\beta) = \beta/(\beta+2c_i)$ on a log-spaced grid
$\beta\in[10^{-2},10^{2}]$ and, for each $\beta$, compute
\[
  R_{\sigma^2}(\beta)
  := \sum_{i=1}^D g\big(s_i^*(M,\beta)\big)
  = C_M(\beta),
  \qquad
  A_{\mathrm{post}}(\beta)
  := \operatorname{Var}_i\big(\log s_i^*(M,\beta)\big).
\]

We then choose a target rate $C_0$ in the overlap of the ranges of
$R_{\sigma^2}(\beta)$ for the two metrics, and select
$\beta_{\mathrm{iso}}$ and $\beta_{\mathrm{aniso}}$ on the grid that minimize
$|R_{\sigma^2}(\beta) - C_0|$ for $M_{\mathrm{iso}}$ and $M_{\mathrm{aniso}}$,
respectively.
\Cref{tab:claim2:toy-results} reports this target $C_0$
(``Target $\sum_i g(s_i^*)$''), the corresponding achieved values
$R_{\sigma^2}(\beta_{\mathrm{iso}})$ and
$R_{\sigma^2}(\beta_{\mathrm{aniso}})$
(``Achieved $\sum_i g(s_i^*)$''), and the posterior anisotropy
$A_{\mathrm{post}}$ at those operating points.
The small discrepancy between target and achieved rate is purely due to the
finite resolution of the $\beta$ grid.
As predicted by Corollary~\ref{cor:claim2:anis} and Theorem~\ref{thm:claim2:toy-main}, the isotropic
metric attains zero anisotropy (up to numerical precision), whereas the
anisotropic metric attains a strictly positive $A_{\mathrm{post}}$ at
essentially the same variance-KL level.

\begin{table}[t]
  \centering
  \caption{Quantitative toy-model summary }
  \label{tab:claim2:toy-results}
  \begin{tabular}{lccc}
    \toprule
    Distortion & Target $\sum_i g(s_i^*)$ & Achieved $\sum_i g(s_i^*)$ & Mean $A_{\mathrm{post}}$ \\
    \midrule
    Pixel MSE ($M=I$) & 17.26 & 17.24 & 0 \\
    Perceptual metric ($M\neq I$) & 17.26 & 17.25 & 0.6486 \\
    \bottomrule
  \end{tabular}
\end{table}

\FloatBarrier

%%%%%%%%%%%%%%%%%%%%%%%%%%%%%%%%%%%%%%%%%%%%%%%%%%%%%%%%%%%%

% \clearpage 
% \input{checklist.tex}

\end{document}